\documentclass{article}

\usepackage[preprint]{neurips_2025}

\usepackage[utf8]{inputenc} %
\usepackage[T1]{fontenc}    %
\usepackage{hyperref}       %
\usepackage{url}            %
\usepackage{booktabs}       %
\usepackage{amsfonts}       %
\usepackage{nicefrac}       %
\usepackage{microtype}      %
\usepackage{xcolor}         %

\usepackage{amsmath,amsfonts,bm}

\def\eqref#1{equation~\ref{#1}}
\def\1{\bm{1}}

\DeclareMathAlphabet{\mathsfit}{\encodingdefault}{\sfdefault}{m}{sl}
\SetMathAlphabet{\mathsfit}{bold}{\encodingdefault}{\sfdefault}{bx}{n}

\newcommand{\KL}{{\mathrm{KL}}}

\DeclareMathOperator*{\argmin}{arg\,min}

\usepackage{graphicx} %
\usepackage{amsmath}
\usepackage{amsfonts}
\usepackage{amsthm}
\usepackage[capitalize,noabbrev,nameinlink]{cleveref}

\usepackage{amsmath}
\usepackage{amsthm}
\usepackage{amssymb}
\usepackage{amsfonts}
\usepackage{graphicx}
\usepackage{caption}
\usepackage{subcaption}
\usepackage{mathrsfs}
\usepackage{wrapfig}
\usepackage{geometry}
\usepackage{multirow}
\usepackage{enumitem}

\usepackage{algorithm2e}
\usepackage[multiple]{footmisc}
\RestyleAlgo{ruled}
\SetKwInOut{Parameters}{parameters}
\crefname{algocf}{alg.}{algs.}
\Crefname{algocf}{Algorithm}{Algorithms}
\DeclareMathOperator*{\crit}{crit}
\newtheorem{theorem}{Theorem}[section]

\newtheorem{lemma}[theorem]{Lemma}
\newtheorem{prop}[theorem]{Proposition}
\theoremstyle{definition}
\newtheorem{definition}{Definition}[section]
\theoremstyle{remark}

\newtheorem{nb}[theorem]{Note}
\newcommand{\Op}[1]{{\mathcal{#1}}}

\newcommand{\mc}[1]{\mathcal{#1}}

\newcommand{\M}{\Op{M}}

\graphicspath{./iclr25/figures}

\title{On Information Geometry and Iterative Optimization \\ in Model Compression: Operator Factorization}

\author{Zakhar Shumaylov$^{1,2}$\thanks{Work done during an internship at Apple.} , Vasileios Tsiaras$^{1}$, Yannis Stylianou$^{1}$ \\
$^1$Apple, $^2$University of Cambridge\\ 
}

\begin{document}

\maketitle

\begin{abstract}
The ever-increasing parameter counts of deep learning models  necessitate effective compression techniques for deployment on resource-constrained devices. This paper explores the application of information geometry, the study of density-induced metrics on parameter spaces, to analyze existing methods within the space of model compression, primarily focusing on operator factorization. Adopting this perspective highlights the core challenge: defining an optimal low-compute submanifold (or subset) and projecting onto it. We argue that many successful model compression approaches can be understood as implicitly approximating information divergences for this projection. 
We highlight that when compressing a pre-trained model, using information divergences is paramount for achieving improved zero-shot accuracy, yet this may no longer be the case when the model is fine-tuned. In such scenarios, trainability of bottlenecked models turns out to be far more important for achieving high compression ratios with minimal performance degradation, necessitating adoption of iterative methods. In this context, we prove convergence of iterative singular value thresholding for training neural networks subject to a soft rank constraint. To further illustrate the utility of this perspective, we showcase how simple modifications to existing methods through softer rank reduction
result in improved performance under fixed compression rates.

\end{abstract}

\section{Introduction}
The last decade in deep learning has witnessed a trend towards increasingly larger models, exemplified by the rise of transformer-based architectures for large language models (LLMs), which have achieved state-of-the-art results in various natural language processing tasks. However, the substantial size of these models poses challenges for deployment on resource-constrained devices, in terms of computational cost, memory footprint, and energy consumption, necessitating model compression as a crucial strategy for mitigating deployment costs.

Although theoretical limits suggest the possibility of arbitrarily narrow models \citep{rochau2024new}, practical experience demonstrates the difficulties in training and generalizing such compact architectures \cite{khodak2021initialization}. Model compression techniques aim to address this challenge by instead training large networks, utilizing the full power of overparameterization \citep{arora2018optimization}, with the aim of reducing the computational and memory footprint of these models post-training, preserving as much performance as possible. Following the categorization presented by \citep{hoefler2021sparsity},  existing compression techniques broadly fall into five categories. Quantization methods reduce the bit-width of weights and activations (e.g. \cite{lin2024awq, tseng2024quip,krishnamoorthi2018quantizing}).
\begin{figure*}[t]
    \centering
    \includegraphics[trim={1cm 0 1cm 0},width=1\textwidth,clip]{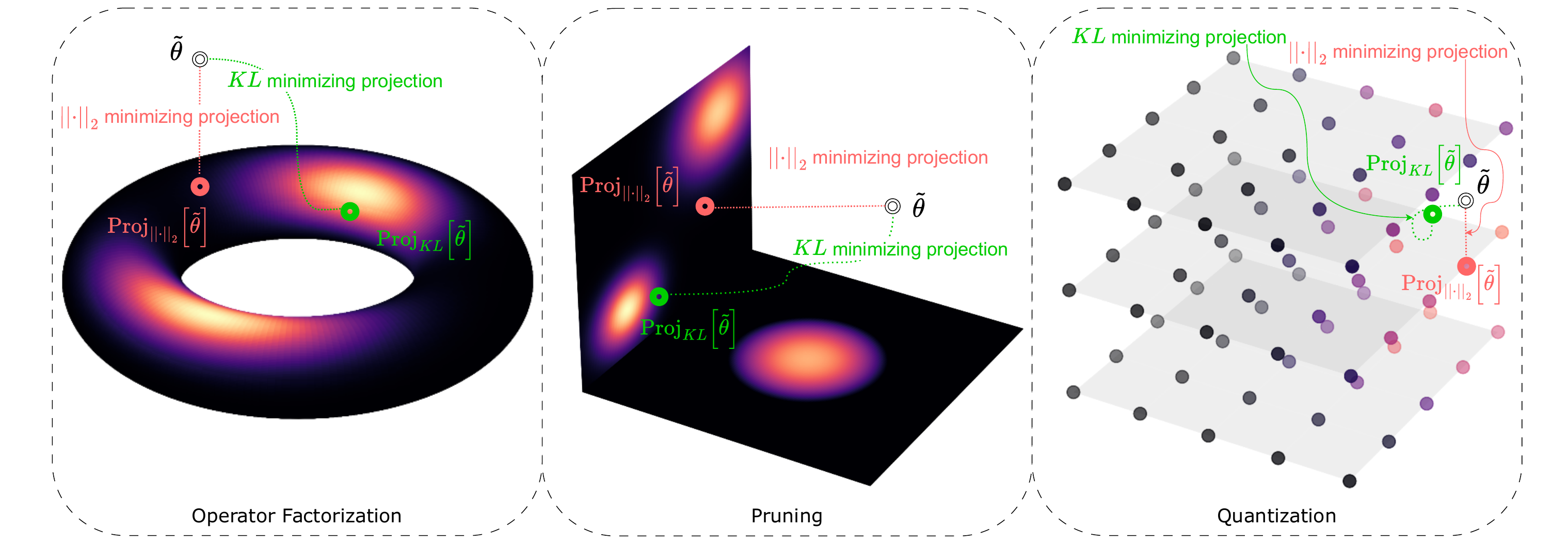}
    \caption{Illustration of the main components of compression methods, highlighted by the main three directions: (1) Operator Factorization, with torus depicting the low rank manifold; (2) Pruning, with two planes depicting the sparse vector space; (3) Quantization, with lattice depicting the level of quantization. On all three, the difference between euclidean versus information projection on the low-compute manifold/set. The color here depicts the loss of the resulting model, highlighting that in a single projection step KL minimizing projection result in significantly improved models.\vspace{-1em}}
    \label{fig:wrapfig}
\end{figure*}
Pruning techniques induce sparsity in weight matrices (e.g. \citep{singh2020woodfisher, Lin2020Dynamic, frantar2023optimal, lecun1989optimal, kim2024shortened}). Distillation methods leverage a teacher-student framework to transfer knowledge to a smaller model (e.g. \citep{hinton2015distillingknowledgeneuralnetwork, mishra2018apprentice, polino2018model}). Parameter sharing methods re-use various network components \citep{uyuk2024learning,plummer2022neural}.
And finally, operator factorization methods decompose weight matrices and tensors into multiple components, which individually are significantly cheaper (e.g. \citep{DBLP:journals/corr/abs-2407-04797,maison2023compression,chen2018groupreduce, noach2020compressing, denton2014exploiting, lebedev2014speeding, sainath2013low,wang2021pufferfish}), often resulting in improved scaling laws \citep{potapczynski2024searching,qiu2024compute}. In this paper we focus on operator factorization, but note that techniques above are often complementary, and combined approaches frequently yield the best compression results \citep{hohman2024model,CALDERA, OATS,zhang2024lqer}, although are non-orthogonal, and must be combined mindfully \citep{harma2025effectiveinterplaysparsityquantization}. Recent years have seen increased theoretical attention devoted to pruning techniques, leading to more theoretically grounded methods with improved performance (e.g. \citep{wu2024the, kuznedelev2023cap, mcgowan2024efficient}). However, within operator factorization, the focus has primarily remained on empirical evaluations (e.g. \citep{FWSVD, TFWSVD, RankDyna, pletenev2023computational}), with largest focus on language modelling tasks. This disparity raises a critical question: can a more rigorous theoretical understanding be developed for operator factorization methods? 

This paper addresses this question directly from two perspectives. Firstly, by establishing a precise mathematical formulation of the problem, we demonstrate that the appropriate framework for understanding model compression lies within information geometry \citep{amari2016information}. Similar observations have previously been highlighted in the context of pruning of natural parameters by \citep{liu2004information}, but has yet to make its way into deep learning. This perspective illuminates why Fisher/Hessian-based formulations are often the most effective, and we explicitly identify the various approximations inherent in existing methods. 
Secondly, we consider the effect of iterative compression methods in analogy with \citep{wu2024the}, highlighting that iterative compression implicitly adjusts the optimization problem, highlighting its connection to iterative singular value hard thresholding. To illustrate the utility of this formulation, we analyze a recently proposed method of \citep{OIALR}, highlighting limitations of the proposed singular value cutoff criterion, and, leveraging our iterative information geometric framework, derive modifications to methods of \citep{OIALR} and \citep{TRP}, through additions of Fisher information.

\section{Preliminaries}

We consider a supervised problem of predicting outputs $y \in \mathbb{Y}$ from inputs $x \in \mathbb{X}$. We assume a probabilistic model for the conditional distribution in the form $p_\theta(y | x)=p(y | f(x, \theta)),$ where $p(y | \cdot)$ is taken to be an exponential family with natural parameters in $\mathbb{F}$ and $f: \mathbb{X} \times \mathbb{R}^D \rightarrow \mathbb{F}$ is a prediction function parameterized by $\theta \in \M =\mathbb{R}^D$. Given $N$ iid training samples $\left(x_n, y_n\right)_{n=1}^N$, we aim to minimize the negative log-likelihood (NLL):
\begin{equation}
\label{eq:loss}
\mathcal{L}(\theta):=-\sum_n \log p_\theta\left(y_n | x_n\right)=-\sum_n \log p\left(y_n | f\left(x_n, \theta\right)\right).
\end{equation}
This framework covers most common scenarios such as least-squares regression with $\mathbb{X}=\mathbb{R}^{\dim\mathbb{X}}, \mathbb{Y}=\mathbb{R}^{\dim\mathbb{Y}}$ with $p(y | f)= \mathcal{N}\left(y ; f, I_{\operatorname{dim}\mathbb{Y}}\right)$ or $C$-class classification with cross-entropy loss $\mathbb{Y}=\{1, \ldots, C\}$, $\mathbb{F}=\mathbb{R}^C$ and $p(y=c | f)=\pi_c(f) := \exp \left(f_c\right) / \sum_i \exp \left(f_i\right)$. Classically, minimizing \Cref{eq:loss} involves using (stochastic) gradient descent. However, adaptive schemes with preconditioning or Gauss-Newton updates, such as natural gradient descent or its approximations yield faster and stabler convergence \cite{duchi2011adaptive, bernstein2024oldoptimizernewnorm, gupta2018shampoopreconditionedstochastictensor,martens2020new}. These methods are able to align with the (information) geometry of the problem by pulling back the distance onbetween parameters via the Kullback-Leibler divergence between the corresponding probability distributions, rather than the Euclidean distance. The intuition can be summarized using the following two facts, implicitly considering $f:\mathbb{X}^N \times \mathbb{R}^D \rightarrow \mathbb{F}^N$ for ease of presentation:
\begin{equation}
\label{eq:loss_as_KL_KL_expansion}
\KL(p_{\theta} | p_{\theta+\Delta\theta}) = \frac{1}{2}\Delta\theta^\top \mathcal{I}(\theta) \Delta\theta + \mathcal{O}(\Delta\theta^3),\qquad\text{and}\qquad\mathcal{L}(\theta)=\KL(p_{\text{em}} | p_{\theta}) + \text{const.}
\end{equation}
for $p_{\text{em}}$ the empirical distribution, $\mc{I}$ denoting the fisher information metric (FIM), coinciding with a generalized Gauss-Newton \cite{schraudolph2002fast} approximation of the Hessian \cite{kunstner2020limitationsempiricalfisherapproximation}:
$$
\mathcal{I}(\theta):=\sum_n \mathbb{E}_{p_\theta\left(y \mid x_n\right)}\left[\nabla_\theta \log p_\theta\left(y | x_n\right) \nabla_\theta \log p_\theta\left(y | x_n\right)^{\top}\right]. 
$$
Thus, for the problems considered the distance induced by KL divergence, represented  infinitesimally using the FIM, appears to be most fitting.

\section{Model Compression}\label{sec:compression}

The central goal of model compression is to derive a \emph{computationally efficient} and \emph{generalizable} model for a given task in \Cref{eq:loss}. The simplest, and seemingly straightforward approach is to directly train a small model. While theory suggests that even extremely narrow networks can approximate any desired function \citep{rochau2024new}, training such networks in practice proves incredibly challenging, if not impossible \cite{waleffe2020principalcomponentnetworksparameter,bejani2020adaptivelowrankfactorizationregularize}. Consequently, the standard practice involves training an overparameterized model, which is then compressed. Though effective, this two-stage process may not actually be the most efficient use of computational resources under a fixed budget \cite{busbridge2025distillationscalinglaws}.

To illustrate the core concepts, we consider a simplified framework for the case of low rank factorization where $\M = \mathbb{R}^{\sum n_i\times n_{i+1}}$ represents the space of weight matrices, and $\M_{<r} = \{\theta\in\mathbb{R}_{<r_i}^{\sum n_i\times n_{i+1}}\}$ denotes the variety of lower-rank matrices, where $r_i$ represents the rank constraint for the $i$-th layer. For a comprehensive discussion of the properties of these sets, see \citep{hiriart2013variational}. 
Considering $\M_{<r}$ as the low compute subset reduces the number of parameters for each layer from $n_i\times n_{i+1}$ in $\mc{M}$ to $r_i\times(n_i+n_{i+1})$, meaning that for sufficiently small $r_i$ these become significantly cheaper, going from a quadratic to a linear number of parameters with respect to dimensions of the matrix. We emphasize that the discussion here generalizes to the setting of pruning and quantization, in which cases the set $\M_{<r}$ would correspond to a linear subspace or lattice correspondingly. Generally, the bilevel problem of interest can be written as 
\[
\min_r \;\|r\|_1 \quad \text{subject to} \quad\min_{\theta\in\M_{<r}} \mathcal{L}(\theta) \leq \varepsilon,
\]
where $\|r\|_1$ represents a measure of the compression level (e.g., number of parameters), $\mathcal{L}(\theta)$ is the loss function, and $\varepsilon$ is an acceptable error threshold. Directly tackling such a bilevel constrained minimization is not computationally feasible by the virtue of $r_i$ being integers and constraint set being non-convex, making the optimization problem not approachable using gradient based methods. Although relaxation based methods exist via masking \cite{DBLP:conf/naacl/GaoHHSJ24}, nuclear norm regularization \cite{TRP} or alternative formulations \citep{wu2024the}. As is common with optimization problems, the problem needs to be approached in an iterative manner. As mentioned above, the most common approach in model compression is to first find a full over-parameterized solution, after which the appropriate $r = r(\tilde\theta)$ is identified and $\tilde\theta$ is projected as ${\theta}_r = \operatorname{proj}_{\M_{<r(\tilde\theta)}}(\tilde\theta)$. As projection may result in performance degradation, the resulting model is fine-tuned:
\[
\tilde\theta = \argmin_{\theta\in\M} \mathcal{L}(\theta) \xrightarrow[]{r(\tilde\theta)}\tilde\theta_r = \argmin_{\theta\in\M_{<r(\tilde\theta)}} \mathcal{L}(\theta).
\]
Such projections can either be done once, resulting in so called \emph{train-then-sparsify} methods, or iteratively, resulting in so called \emph{sparsify-during-training} methods \cite{hoefler2021sparsity}. As long as the minimum is on the restricted set, it should be possible to find it, as $\M_{<r_i}$ is simply connected \cite{uschmajew2020geometric}.

To summarize, any model compression algorithm must address two main questions. That is \textbf{rank selection}, and \textbf{projection}. In \Cref{sec:related}, we summarize existing approaches in the literature surrounding operator factorization, showcasing that their heuristic success can be attributed to improvements in either rank or projection. Based on the discussion in the previous section, we argue that projection must be performed according to information distance, with rank selection according to the resulting distance, in terms of KL divergence, to the projection. 
\subsection{Information Geometry in Model Compression}\label{sec:ig}
Following the framework of \citep{amari2010information}, we begin by considering a family of probability distributions $\mathcal{D} = \{p_\theta(\cdot) \mid \theta \in \mc{M} \subseteq \mathbb{R}^d\}$, parameterized by $\theta \in \mc{M}$. This family forms a statistical manifold, with $\theta$ serving as a local coordinate system. This manifold is then equipped with a \emph{divergence} $D(p_\theta \| p_{\theta'})$, quantifying dissimilarity between distributions $p_\theta$ and $p_{\theta'}$ on $\mathcal{D}$. In general, this divergence does not need to be symmetric or satisfy the triangle inequality. An important class of divergences in this context are the \emph{Bregman divergences} generated by a strictly convex and differentiable function $\phi$:  
\begin{equation}\label{eq:bregman_div}
    D_\phi(\theta \| \theta^\prime) = \phi(\theta) - \phi(\theta^\prime) - \nabla \phi(\theta)^\top (\theta - \theta^\prime).
\end{equation}
For exponential and mixture families, the \emph{Kullback-Leibler (KL) divergence} $\operatorname{KL}(p_\theta | p_{\theta'}) = \mathbb{E}_{p_\theta}[\log(p_\theta(x) / p_{\theta'}(x))]$ arises naturally and can be viewed as a Bregman divergence, generated by the entropy or log-partition function \citep{amari2016information}. The local structure of the statistical manifold arises from this divergence.
Specifically, considering an infinitesimal perturbation $\mathrm{d}\theta$ around a point $\theta$, the second-order Taylor expansion of the divergence $D(p_\theta \| p_{\theta + \mathrm{d}\theta})$ defines the \emph{Riemannian metric tensor} $g(\theta)$ with components:
$
g_{ij}(\theta) = \frac{\partial^2}{\partial \theta^i \partial \theta'^j} D(p_\theta \| p_{\theta'})|_{\theta' = \theta}.
$
This metric endows $\mathcal{D}$ with a Riemannian structure, allowing us to define notions of distance and orthogonality. As is often considered in information geometry \citep{ay2017information}, we are also interested in dual connections generated by the divergence \citep{amari2010information}: 
$$
\Gamma_{i j k}(\theta)=-\frac{\partial^3}{\partial \theta_i \partial \theta_j \partial \theta^\prime_k} D(\theta\|\theta^\prime)_{\mid \theta=\theta^\prime}, \quad
\Gamma_{i j k}^*(\theta)=-\frac{\partial^3}{\partial \theta^\prime_i \partial \theta^\prime_j \partial \theta_k} D(\theta\|\theta^\prime)_{\mid \theta=\theta^\prime},
$$
which are dually coupled with respect to $g_{i j}$ \citep{amari2000methods}. The dual connections $\Gamma$ and $\Gamma^*$ result in two families of geodesics, termed the \emph{$e$-geodesics} and \emph{$m$-geodesics}, respectively. These play a crucial role in defining projections onto submanifolds within the statistical manifold.

\paragraph{Iterative \textit{m}-projection for Pruning in Flat Manifolds}
A manifold in coordinates $\theta$ is \emph{$e$-flat} (exponential-flat), if $\Gamma_{i j k}=0$, and \emph{$m$-flat} (mixture-flat) if $\Gamma^*_{i j k}=0$. A typical example of an $e$-flat manifold is the exponential family:
$
p(x, f)=\exp \left\{\sum f_i k_i(x)-\psi(f)\right\},
$
where $k_i(x)$ are given functions and $\psi$ is the normalizing factor. In such manifolds, coordinate pruning, as in \citep{liu2004information}, can be viewed as iterative \emph{$m$-projection} and analyzed as follows. Suppose we are given two flat submanifolds of $\mc{M}$ with $p_1 \in E_1 \subset E_2$. Then for $p_2^*$ the $m$-projection of $p_1$ onto $E_2$,
$p_2^* = \argmin_{p\in E_2}\;D(p_1\|p),$ 
we have that $p_2^*$ is given by the $m$-geodesic connecting $p_1$ and $p_2^*$ orthogonal to $E_2$ at $p_2^{*}$ (e.g. by \Cref{prop:orthog} or \citep{liu2004information}). Then, for any $p_2 \in E_2$, the Generalized Pythagorean theorem \citep{ay2017information,nielsen2021geodesictrianglesrightangles} results in:
$
D(p_1 \| p_2) = D(p_1 \| p_2^*) + D(p_2^* \| p_2).
$
This allows for direct analysis of error accumulation under iterative projection onto lower-dimensional (pruned) flat submanifolds.

\paragraph{Iterative \textit{m}-projection in Curved and Overparameterized Manifolds}

In the case of overparameterized models, the mapping $f: \mc{M} \to \mathbb{F}$, the image $f(\mc{M})$ may not be a simple submanifold, potentially containing non-differentiable structures, making theory for curved exponential families \citep{amari2016information,amari1995information} unusable. While the generalized pythagorean theorem may still be used when the image is convex \citep{kieffer1994elements}, the case of deep neural networks does not fall into any of these categories. Luckily, for the problem of model compression, the global structure of the parameterization is not of interest, only the local structure is. This is further helped by the fact that for operator factorization \citep[Example 8.14]{lee2003smooth} and pruning (but not quantization), the low-compute subset is a smooth manifold, implying that the generalized pythagorean theorem can be used as long as locally, the parameterization image results in a submanifold \cite[Corollary 4.2]{ay2017information}. How can this be achieved? 

We consider an open ball around the solution we found $B(\tilde\theta)$ such that intersection with the low rank manifold $\mc{M}_{r}$ is non-empty. Then, it is open and a submanifold \citep{lee2003smooth}.
We can then consider the set $f(\mc{M}_{r} \cap B(\tilde\theta))$. Since $\mc{M}_{r} \cap B(\tilde\theta)$ is an open subset of $\mc{M}_{r}$, it is a submanifold and therefore a manifold. Thus considering $f|_{B(\tilde\theta)}$, assuming it is a submersion, by \Cref{lem:submersion} below we know that $f(\mc{M}_{r} \cap B(\tilde\theta))\in \mathbb{F}$ is a submanifold, establishing information projections through geodesics:
\begin{prop}[Corollary 4.2 \cite{ay2017information}]\label{prop:orthog}
 Let $N$ be a differentiable submanifold of $\mc{M}$. Then $q \in N$ is a stationary point of the function $D(p \| \cdot): N \rightarrow \mathbb{R}, r \mapsto D(p \| r)$ iff the $m$-geodesic from $p$ to $q$ meets $N$ orthogonally.    
\end{prop}

\subsection{Related Works in Operator Factorization}\label{sec:related}

\paragraph{SVD, TRP, DLRT}\label{sec:euclidean}
The first class of approaches are based on Euclidean projection of matrices onto their low-rank variants, which by Eckart Youngs Theorem is equivalent to SVD:
\begin{equation}
\operatorname{proj-\|\cdot\|_2}_{\M_{<r}}(\tilde\theta) = \argmin_{\theta\in\M_{<r}} \|\theta-\tilde\theta\|^2_2.    
\end{equation}
There have been a multitude of papers on various variants of such a technique starting with \cite{psichogios1994svd} and many since \cite{maison2023compression,chen2018groupreduce, noach2020compressing, denton2014exploiting, lebedev2014speeding, jaderberg2014speeding, sainath2013low}. 
This approach has been successful in compressing models and sometimes even improving generalization by reducing overfitting during training \cite{winata2020lightweight,phan2020stable}.

These approaches have also been adapted to be used during training via iterative projection \cite{TRP, LRT, OIALR}.
The main benefit of using SVD directly is that minimizers can be found efficiently separately per layer. 
For classical SVD, the rank has to be provided in advance, or similar to iterative methods can be chosen based on the values directly. However, beyond heuristics it is not clear how to choose the rank. Both the distribution and size of singular values differs significantly amongst layers \cite{DBLP:journals/corr/abs-2407-04797} and in transformers is significantly different between attention and MLP weights \citep{OATS,zhang2025lowrank}.
\paragraph{FWSVD, TFWSVD, RankDyna}\label{sec:fisher}
The next class of approaches we are going to consider is based precisely on the discussion of \Cref{sec:compression}, with methods approximating the information projection:
\begin{align}\label{eq:KL_projection}
 \operatorname{proj-KL}_{\M_{<r}}(\tilde\theta) &= \argmin_{\theta\in\M_{<r}} \;\operatorname{KL}\left(p_{\tilde\theta} \mid p_{\theta}\right) \approx \argmin_{\theta\in\M_{<r}} \frac{1}{2}(\theta-\tilde\theta)^\top \mathcal{I}(\tilde\theta) (\theta-\tilde\theta) + \mathcal{O}((\theta-\tilde\theta)^3). 
\end{align}
However, computing the full FIM is computationally prohibitive due to its quartic complexity with respect to the weight matrix size. Thus, approximations must be used. Methods like Fisher-Weighted SVD (FWSVD) \citep{FWSVD} and its variants like TFWSVD \citep{TFWSVD}, along with the iterative projection method RankDyna \citep{RankDyna}, implicitly approximate this KL projection, although they often frame the FIM as representing parameter ``importance'' \citep{molchanov2019importance} related to the loss function.  These approaches can be viewed as attempting to minimize the KL divergence under significant approximations: TFWSVD diagonalizes the fisher information $\mathcal{I}(\theta)\approx\operatorname{diag}\;\mathcal{I}(\theta):=\widehat{\mathcal{I}}$, decoupling layers and discarding potentially crucial intra-layer interactions. But even with a diagonal FIM, the resulting weighted SVD problem remains computationally expensive \citep{srebro2003weighted}. To address this, FWSVD further approximates the diagonal FIM by summing the diagonal elements column-wise, making the approximation identical for each row: 
$\widehat{\mathcal{I}}(\theta_k)_{ij} \approx \operatorname{diag}\sum_j\widehat{\mathcal{I}}(\theta_k)_{ij}:={\widetilde{\mathcal{I}}}(\theta_k)_{ij}$ for layer $k$ parameters $\theta_k = \operatorname{vec}{W^k}$. As in the Euclidean case, the rank has to be provided in advance or chosen based on the values of the singular values, although selection criteria do exist \citep{hofstee2024fisher}. Analogously, importance based weighting \citep{molchanov2019importance} can be described through this lens.

As the overarching goal is to minimize the loss $\operatorname{KL}\left(p_{\text{em}} \mid p_{\theta}\right)$, we may attempt to instead expand the expression above around $\tilde\theta$, proposed by \citep{molchanov2019importance}. While at the minimum the first order term is zero, in practice it is not exactly equal to zero and the approximation itself may not be accurate. 
\begin{align}\label{eq:importance}
 \operatorname{proj-KL_{em}}_{\M_{<r}}(\tilde\theta) &= \argmin_{\theta\in\M_{<r}} \;\operatorname{KL}\left(p_{\text{em}} \mid p_{\theta}\right) \\ &\approx \argmin_{\theta\in\M_{<r}}  \;(\theta-\tilde\theta)^\top \nabla_\theta\operatorname{KL}\left(p_{\text{em}} \mid p_{\theta}\right) \nonumber \\
 &\qquad\qquad+ \frac12(\theta-\tilde\theta)^\top \mathcal{I}_{\text{em}}(\tilde\theta) (\theta-\tilde\theta)\nonumber +  \mathcal{O}((\theta-\tilde\theta)^3). \nonumber
\end{align}

\paragraph{ASVD, CALDERA, OATS}
Due to the inherent cost of evaluating the FIM, many works \cite{ASVD,SVDLLM,OATS} have instead focused on ensuring agreement of intermediary (post-) activations, turning it into a proxy for the full information distance. Distributions of intermediary activations are quite particular, oftentimes possesing significant outliers \citep{OATS,ASVD}. If we denote $\theta_i\in\mathbb{R}^{n_i\times n_{i+1}}$ and $X_{i-1}$ as precomputed post-activations or some statistic of post-activations, the approach is to minimize the average reconstruction error, which from the point of view of information geometry corresponds to finding projection under per-layer gaussian density assumption, with $p^{\mathcal{N}}_{\theta}(\cdot) = \prod_i{\mathcal{N}}(\cdot | X_i, I_{n_i\times n_i})$:
\begin{equation}
\operatorname{proj-KL}^{\mathcal{N}}_{\M_{<r}}(\tilde\theta) = \argmin_{\theta\in\M_{<r}} \sum_i \|\theta_i X_{i-1}-\tilde\theta_i X_{i-1}\|^2_2 = \argmin_{\theta\in\M_{<r}}\;\operatorname{KL}\left(p^{{\mathcal{N}}}_{\theta} \mid p^{{\mathcal{N}}}_{\tilde\theta}\right)   
\end{equation}
Taking on such a view raises a question: is such a gaussian assumption the right one? Oftentimes post-activations cluster around multiple points (e.g. for different classes), or may contain outliers. As before, the rank has to be provided in advance or chosen based on the value of the singular values. Other heuristics also exist, e.g. \cite{ASVD} performs an expensive binary search for each rank depending on the accuracy drop. 

\begin{figure}[t!]
    \centering
    \begin{subfigure}[b]{0.48\textwidth}
        \centering
        \includegraphics[width=\textwidth]{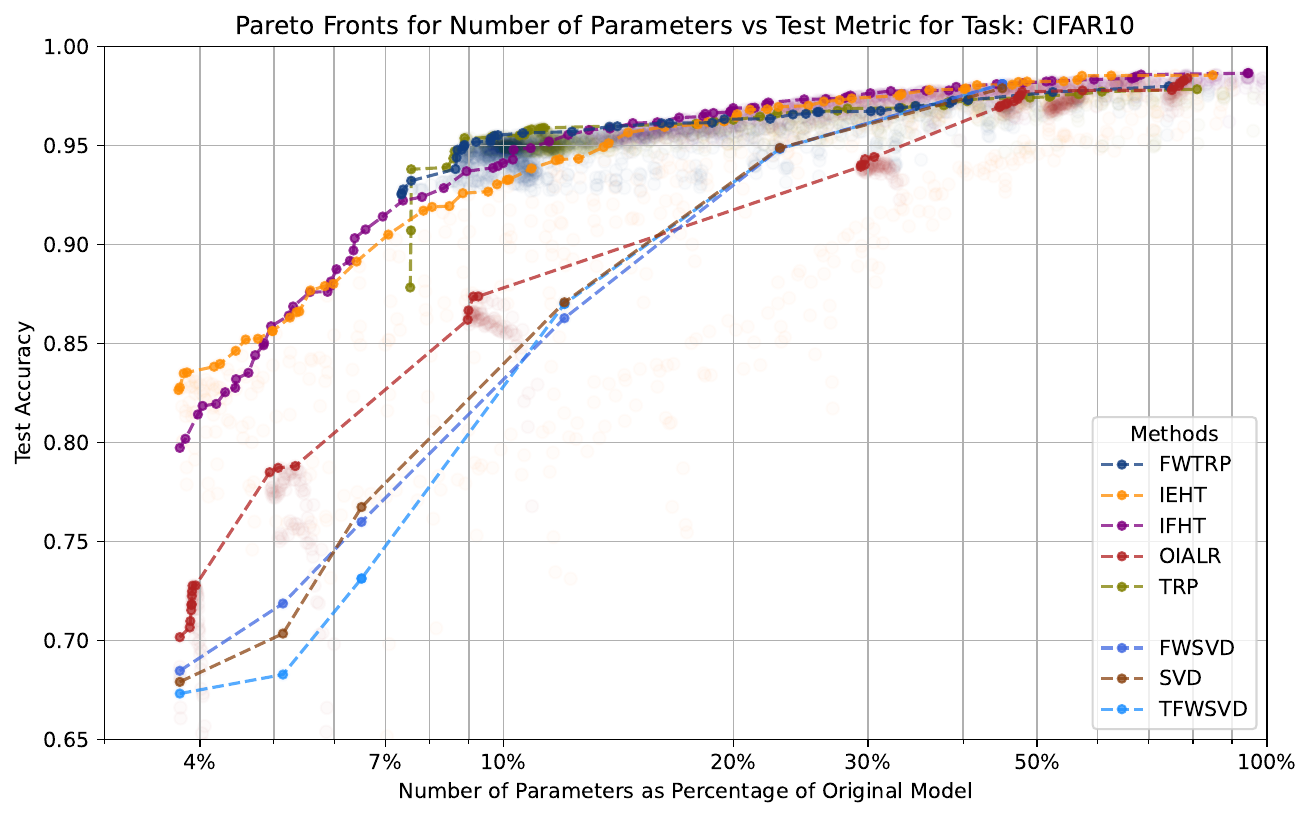} %
        \caption{CIFAR-10 classification.}
        \label{fig:cifar10}
    \end{subfigure}
    \hfill
    \begin{subfigure}[b]{0.48\textwidth}
        \centering
        \includegraphics[width=\textwidth]{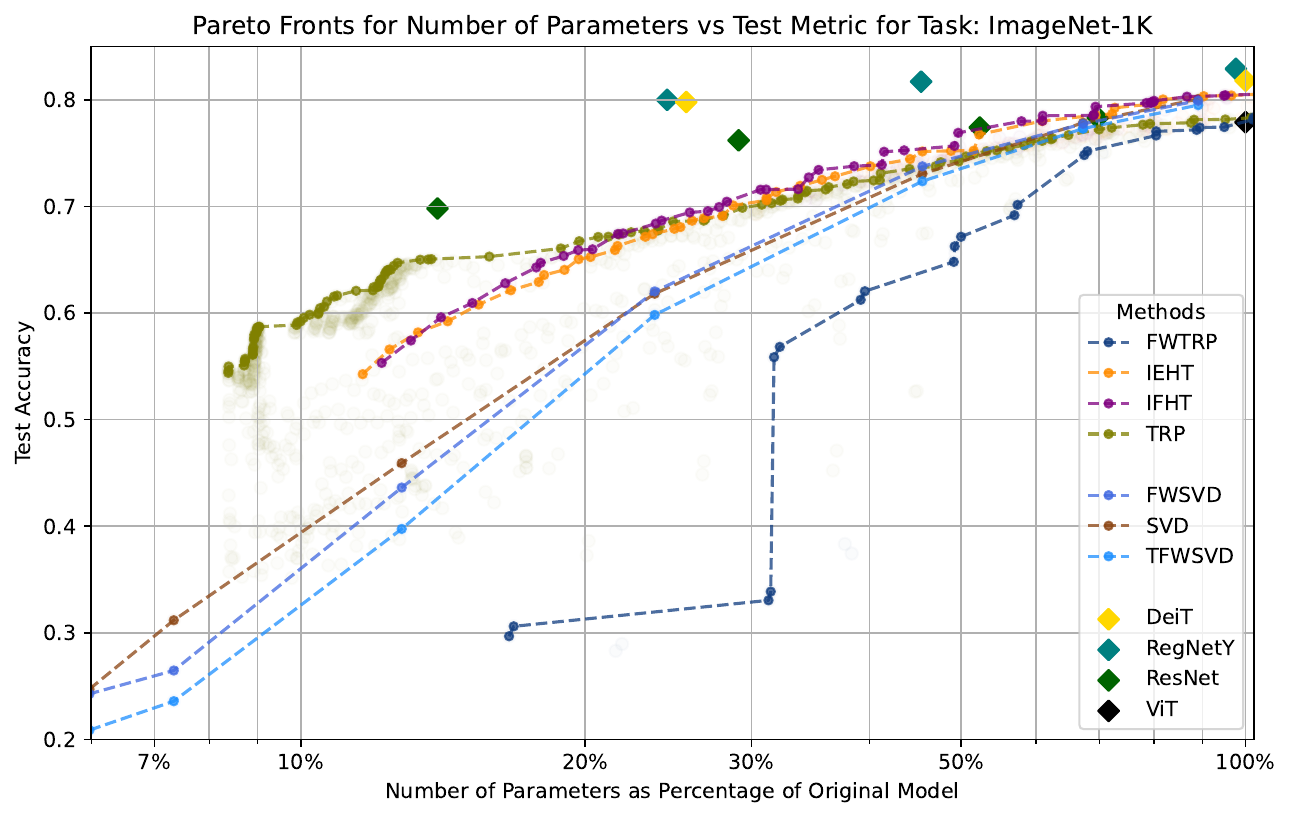} %
        \caption{ImageNet-1k classification.}
        \label{fig:imagenet}
    \end{subfigure}
    \caption{Pareto fronts for compression methods in \Cref{sec:compression} in for ViT-B/16 for image classification. Diamonds show performance reported by \citep{touvron2021trainingdataefficientimagetransformers}.}
    \label{fig:paretos_img_class}
\end{figure}

\section{Sparsify During Training as Iterative Hard Thresholding}\label{sec:IHT}
In the previous section we discussed information geometric aspects involved in the \emph{projection} step. However, the question \emph{rank selection} is addressed only heuristically. Based on the observation that training/finetuning bottlenecked networks to good accuracy is not always possible, we are motivated to consider \emph{sparsify-in-training} methods, iteratively reducing the rank. 
\paragraph{The Deep Linear Network Example}
To illustrate the rationale behind iterative rank reduction, we consider the case of deep linear networks.  While simple, these models have proven valuable for studying overparameterization, revealing insights into implicit biases and optimization benefits \citep{menon2024geometrydeeplinearnetwork}.  A key observation in overdetermined problems is the tendency for solutions to exhibit low-rank structures, a phenomenon captured by analyzing the entropy of low-rank solutions within the solution space, with low-rank solutions possessing significantly higher entropy. However, convergence to such a low-rank minimum is not guaranteed \citep{cohen2023deep}, even if they are more likely. Therefore, iterative rank reduction methods can be interpreted as a means of guiding the optimization process towards these desirable low-rank minima, effectively ``massaging'' the solution towards a compressed representation.

\paragraph{Iterative Hard Thresholding in Low Rank Approximation}
Based on the observations, that iterative rank reduction \citep{OIALR,LRT} is a projection onto the low-rank manifold and that hard thresholding of singular values is a proximal step on the rank function, we consider the resulting minimization problem and proximal gradient method used to optimize it. Similar observations, have been made in the context of pruning \citep{wu2024the}.
Low-rank approximation is a fundamental problem in various fields with a rich literature of algorithms to address it, e.g. in the context of matrix completion \citep{Vu_2022}. Because the rank of a matrix is not a continuous function, many methods relax this constraint to allow for continuous optimization, including nuclear norm minimization \citep{candes2012exact,rao2015forward} as a convex relaxation; via low-rank factorization \citep{jain2013low,sun2016guaranteed}, resulting in a non-convex least-squares problem; or considering a rank constrained formulation approached via singular value thresholding \citep{chunikhina2014performance,tanner2013normalized}.

While the methods are effective, they often operate under simplifying assumptions on the objective, which make them non-applicable to the setting of neural networks, necessitating an extension of the existing analyses.
The problem we are going to consider here is going to be of the form (for the sake of simplicity here presented as for a single matrix, but extends naturally beyond):
\begin{equation}\label{eq:rank_regularized}
\min_{W\in \mc{M}} \;\;\mathcal{L}(W) + \lambda \operatorname{rank}(W).
\end{equation}
Now, the classical proximal gradient method for \Cref{eq:rank_regularized} consists of iterative steps of the form 
$W_{n+1} = \operatorname{prox}_{\alpha_n\lambda\operatorname{rank}(\cdot)}\left(W_{n}-\alpha_n \nabla \mc{L}\left(W_n\right)\right).
$
The power of the proximal method comes from the fact that normally the proximal operator of a function is expensive to compute. Intuitively proximal operator of rank involves projection on the lower rank manifold, which is extremely non-convex and more so has empty interior. However, thanks to Eckart Youngs Theorem we have a closed form way of performing such a projection \citep{hiriart2013variational,hiriart2013eckart}. More generally however, we are going to consider a (mirror) proximal gradient method, with divergences (\Cref{eq:bregman_div}), instead of the $l_2$ norm, with updates of the form 
\begin{equation}\label{eq:step_prox}
    W_{n+1} \in \underset{W \in \mathcal{M}}{\operatorname{argmin}}\left\{\frac{1}{\alpha_n}D_{F_n}\left(W, W_n\right)+\left\langle W, \nabla \mc{L}\left( W_n\right)\right\rangle+\lambda\operatorname{rank}(W)\right\} .
\end{equation}
Motivation for such updates is based on discussion of \Cref{sec:fisher}: information induced distances are of more interest. Ideally, instead of $D_{F_n}\left(W, W_n\right)$ we would consider $\operatorname{KL}\left(p(W_n) \;|\; p(W)\right)$, however this is captured sufficiently well by the FIM, for which $F_n = \frac12\|\cdot\|^2_{\mc{I}(W_n)}$.

\subsection{Proximal Gradient Algorithm Analysis}
In this section, we consider the problem in \Cref{eq:rank_regularized} and analyze the iterative scheme of \Cref{eq:step_prox}. 
To show convergence, we assume the following: 
\begin{itemize}[nosep,leftmargin=*]
    \item $\mc{L}$ is a bounded from below Fréchet differentiable function with Lipschitz continuous gradient, i.e. there exists $L_{\nabla_\mc{L}} \geq 0$ such that $\|\nabla \mc{L}(W)-\nabla \mc{L}(W^\prime)\| \leq L_{\nabla_\mc{L}}\|W-W^\prime\|$. Further assume that $\mc{L}$ is sub-analytic and coercive. 
    \item Each $F_n: \mathbb{R}^m \rightarrow \mathbb{R}$, assumed to be $\sigma_n$-strongly convex, Fréchet differentiable and such that $\nabla F_n$ is $L_{\nabla F_n}$-Lipschitz continuous with $L_{\nabla F_n}\leq\overline{L_{\nabla F}}$. %
\end{itemize}
\begin{nb}
The lipshitzness assumption on $\mc{L}$ is not necessary, and can be generalized to a local version \cite{jia2023convergence}. Coercivity is not restrictive, as oftentimes weight decay is used. Sub-analyticity condition on $\mc{L}$ is not restrictive, as by e.g. \cite[Theorem E.3]{shumaylov2024weakly}, if $\mc{N}$ is a neural network with continuous, piecewise analytic activations with finitely many pieces (e.g., \texttt{ReLU,sigmoid}), then $\mc{N}$ is sub-analytic.
\end{nb}
\begin{lemma}[\citep{hiriart2013variational}]\label{lem:rank}
    The rank function $\operatorname{rank}(W)$ is proper, lower semicontinuous, bounded from above and below. Furthermore, it is semi-algebraic and therefore K\L. 
\end{lemma}
\begin{prop}[Proof in \Cref{sec:proof_prop}]\label{prop:convergence}
    Assume that for all $n$, $0<\underline{\alpha}\leq\alpha_n\leq\frac{\sigma_n}{L_{\nabla_{\mc{L}}}}$ and assumptions above hold. Then, the following are true:
    \begin{itemize}
        \item $\left(\mathcal{L} + \lambda \operatorname{rank}\right)(W_n)$ is non-increasing and convergent.
        \item $\sum_n\|W_{n+1}-W_{n}\| < +\infty$.
        \item $W_n$ converges, with $\lim_{n\to\infty} W_n = W^*\in {\operatorname{crit}{\left(\mathcal{L} + \lambda \operatorname{rank}\right)}}$.
        \item For $F_n(W) = \frac12\|\mc{I}_n^{1/2}W\|^2_2$ with $\mc{I}_n\to \mc{I}$ positive definite, the smallest non-zero singular value satisfies $\sigma_{\text{min}}\left(\mc{I}^{1/2}W^*\right)\geq \sqrt{\underline{\alpha} \lambda}$. And thus $\sigma_{\text{min}}\left(W^{*}\right)  \geq \sqrt{\sigma_{\text{max}}\left(\mc{I}\right)\underline{\alpha} \lambda}$.
    \end{itemize}
\end{prop}

\Cref{prop:convergence} establishes convergence and a maximal rank of the underlying solution. However, in practice using such a method implies computing SVD of all the weights at each iteration, making it prohibitively expensive in practice for large networks. We can instead provide a different claim for e.g. \cite{OIALR} in \Cref{sec:app_oialr_theory}.
\section{Numerical Illustrations}
While theoretical results of \Cref{sec:IHT} highlight that the proposed iterative thresholding approach is coherent, it does not provide us with any guidelines on how to select the step-sizes $\alpha_n$, or provide us with information of whether using information divergences should result in better models. For this reason, we consider what happens numerically, asking the question of what has the largest effect on compressed model accuracy. We summarize all the methods considered for comparison in \Cref{tab:method_comparison}, and further experimental details are provided in \Cref{ap:details}.

\subsection{Vision Transformer}\label{sec:exp_vit}
For our initial experiments, we focused on compressing a pre-trained Vision Transformer (ViT)-B/16 (86.6M) model \citep{dosovitskiy2021imageworth16x16words} trained on CIFAR10 \citep{krizhevsky2014cifar} and one trained on ImageNet-21k \citep{ridnik2021imagenet21kpretrainingmasses,5206848} at a resolution of 224x224. We then finetune during compression on ImageNet 2012 \citep{russakovsky2015imagenetlargescalevisual}.
\paragraph{Information Projection: Zero-shot vs Fine-tuned}
\begin{table}
    \centering
    \caption{Showcasing compression algorithms applied to MLP layers based on information and standard projections for the setting of \Cref{sec:exp_vit}. Using information projections is necessary for zero-shot performance, while fine-tuned performance is nearly identical for the two methods.}
    \label{tab:combined}
    \resizebox{.95\linewidth}{!}{%
    \begin{tabular}{cccccccc|c}
        & \% of parameters & 35.7\% & 36.6\% & 37.4\% & 40.9\% & 47.8\% & 61.5\% & Full Model \\ \hline
        \multirow{2}{*}{Zero-shot} & 
        SVD & 0.0997 & 0.1005 & 0.1054 & 0.1302 & 0.4126 & 0.8601 & \multirow{2}{*}{0.9562} \\ 
        & FWSVD & \textbf{0.1223} & \textbf{0.1320} & \textbf{0.1199} & \textbf{0.3383} & \textbf{0.6982} & \textbf{0.9147} & \\ \hline
        \multirow{2}{*}{Fine-tuned} & SVD & 0.9237 & 0.9279 & 0.9450 & \textbf{0.9656} & 0.9585 & \textbf{0.9809} & \multirow{2}{*}{0.9878}\\ 
        & FWSVD & \textbf{0.9336} & \textbf{0.9290} & \textbf{0.9541} & 0.9446 & \textbf{0.9736} & 0.9581 & \\ \hline
    \end{tabular}}
\end{table}
\Cref{tab:combined} shows the zero-shot and fine-tuned accuracy results for compressing the MLP layers of the ViT-B/16 model using fixed rank SVD and FWSVD \citep{FWSVD}, for various ranks. The results indicate that employing information-based projections (FWSVD) is crucial for preserving zero-shot performance post-compression. However, once the models are fine-tuned, the performance gap between SVD and FWSVD becomes negligible, although FWSVD in more cases than not still exhibits a slight advantage. This suggests that fine-tuning effectively mitigates the errors from Euclidean projection. We similarly observed this trend across variable rank methods, depicted in \Cref{fig:paretos_img_class}, including the IEHT and IFHT discussed in \Cref{sec:IHT}. This finding underscores that in training information projections have an insignificant effect.
\paragraph{Trainability as the deciding factor for Extreme Compression}
Based on the discussion above, information projection has a significant effect only in zero-shot settings. Previous literature has established that initiating training with low-rank models from the outset hinders effective training \citep{khodak2021initialization}. We hypothesize that this is the most limiting factor for achieving good performance of compressed models. As illustrated in \Cref{fig:paretos_img_class}, there is a significant gap between methods employing iterative compression, and those not. To highlight the role of trainability, we conducted two further ablations:
\paragraph{Sparsify-During-Training Methods:}
    We employed OIALR \cite{OIALR}, as shown in \Cref{fig:cifar10}. While OIALR can recover improved models for the same compression ratios as constant-rank methods, it suffers from non-smooth rank reduction, often involving a large initial cut followed by several smaller ones. This limits the recovery potential of the models. To address this, we implemented the IFHT and IEHT methods from Section \ref{sec:IHT}, which exhibit less severe rank cutoffs.
\begin{figure}[t!]
    \centering
    \begin{subfigure}{0.49\textwidth}
        \centering
        \includegraphics[width=\textwidth]{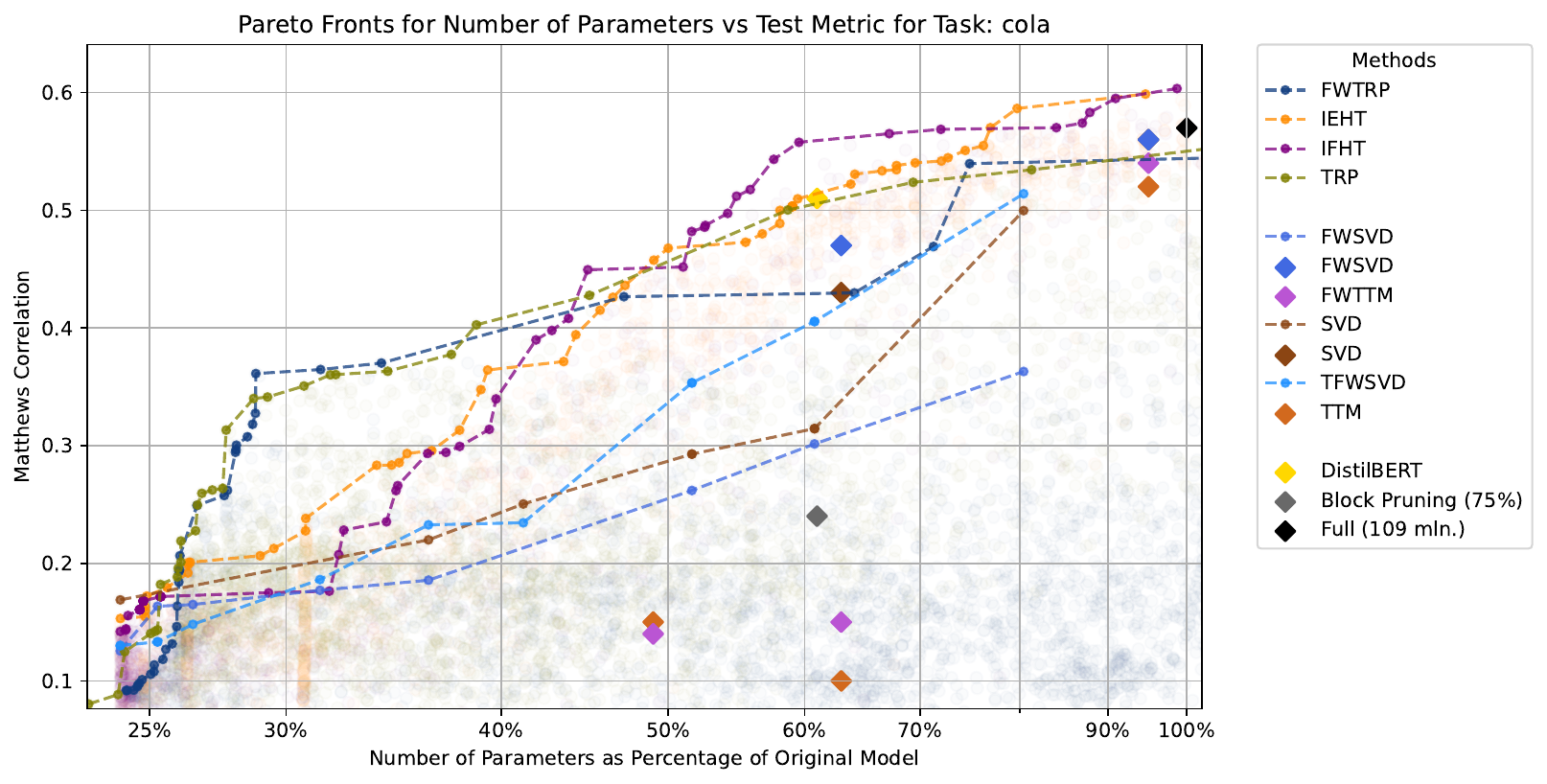}
        \caption{COLA}
    \end{subfigure}
    \hfill
    \begin{subfigure}{0.49\textwidth}
        \centering
        \includegraphics[width=\textwidth]{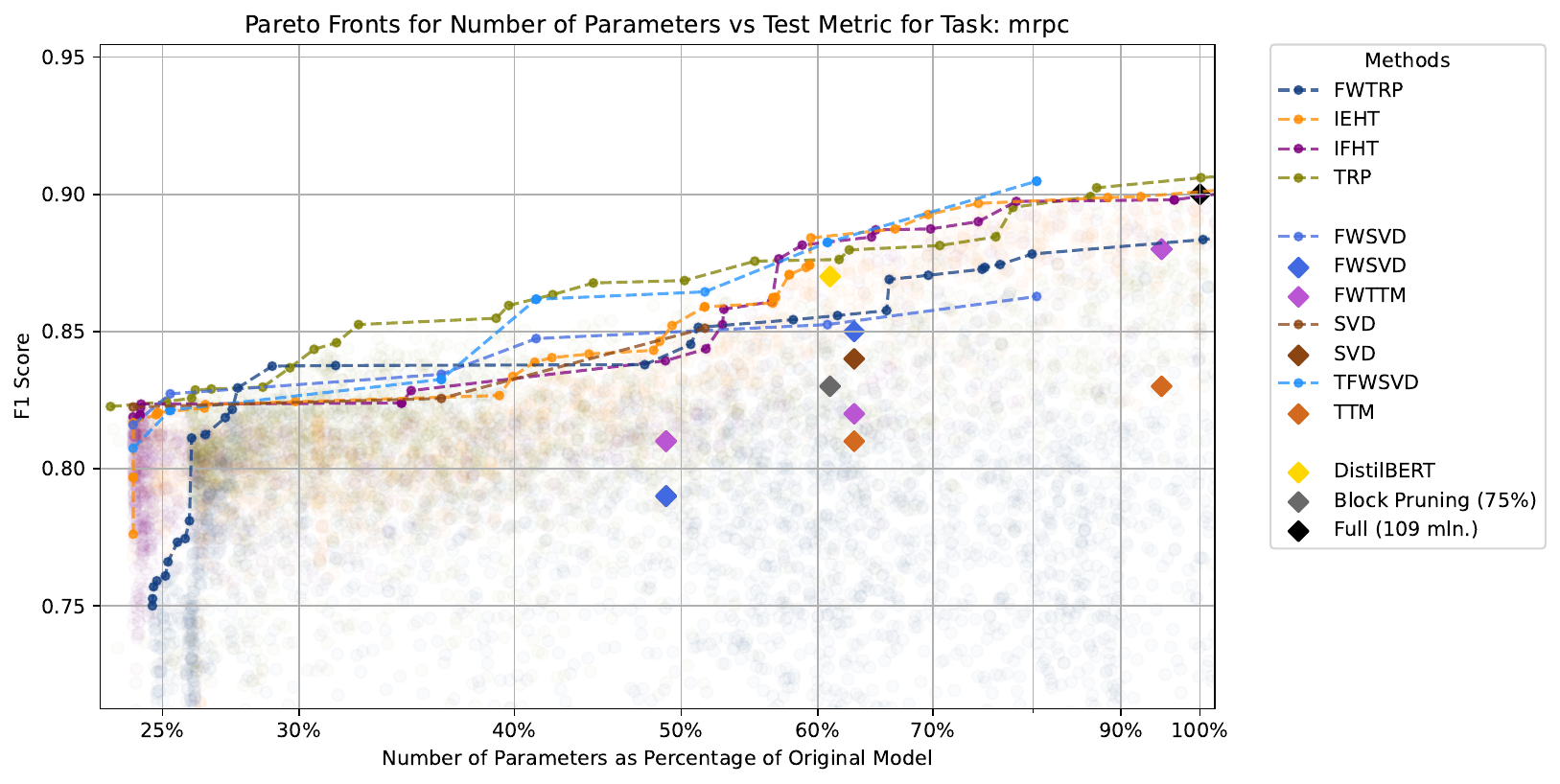}
        \caption{MRPC}
    \end{subfigure}
    \\
    \begin{subfigure}{0.49\textwidth}
        \centering
        \includegraphics[width=\textwidth]{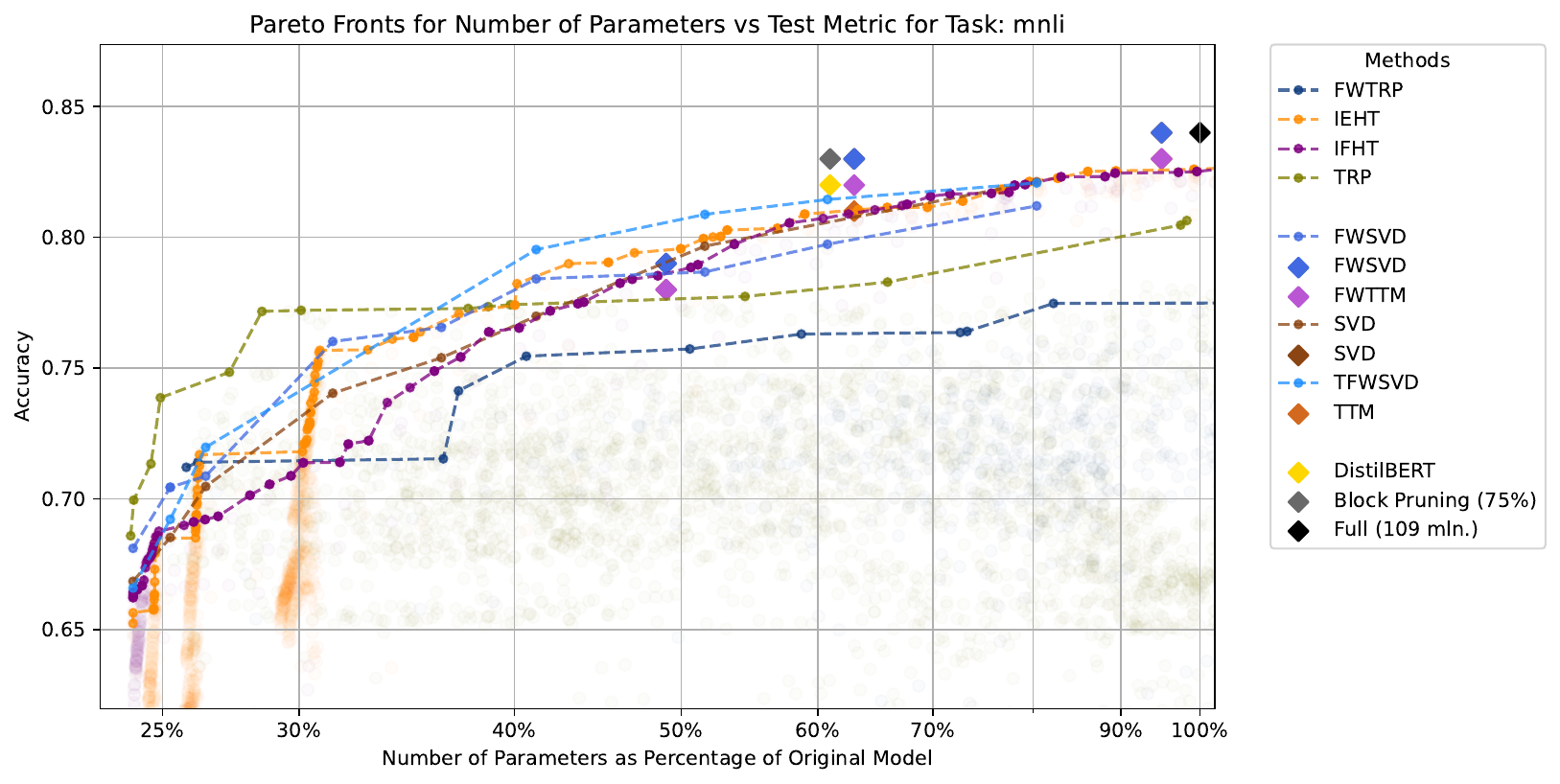}
        \caption{MNLI}
    \end{subfigure}
    \hfill
    \begin{subfigure}{0.49\textwidth}
        \centering
        \includegraphics[width=\textwidth]{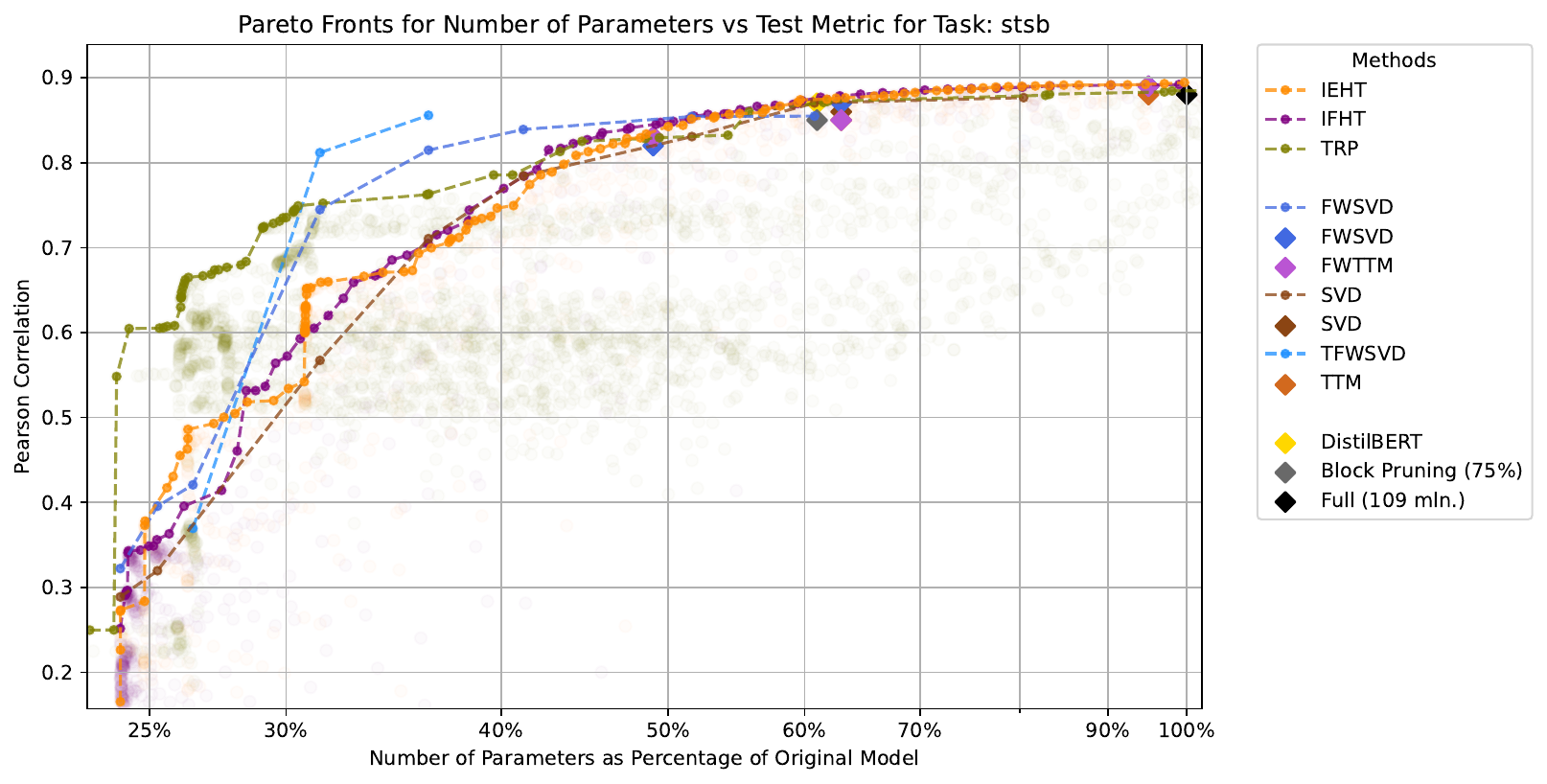}
        \caption{STSB}
    \end{subfigure}
    \caption{Performance of various compression algorithms on BERT for tasks in the GLUE benchmark\label{fig:BERT}. Diamonds show performance reported by \citep{pletenev2023computational}.}
\end{figure}
\paragraph{Different Cutoff Schedules:}

\begin{wrapfigure}{O}{0.48\textwidth} %
\vspace{-1.4em}
    \centering
    \includegraphics[width=.48\textwidth]{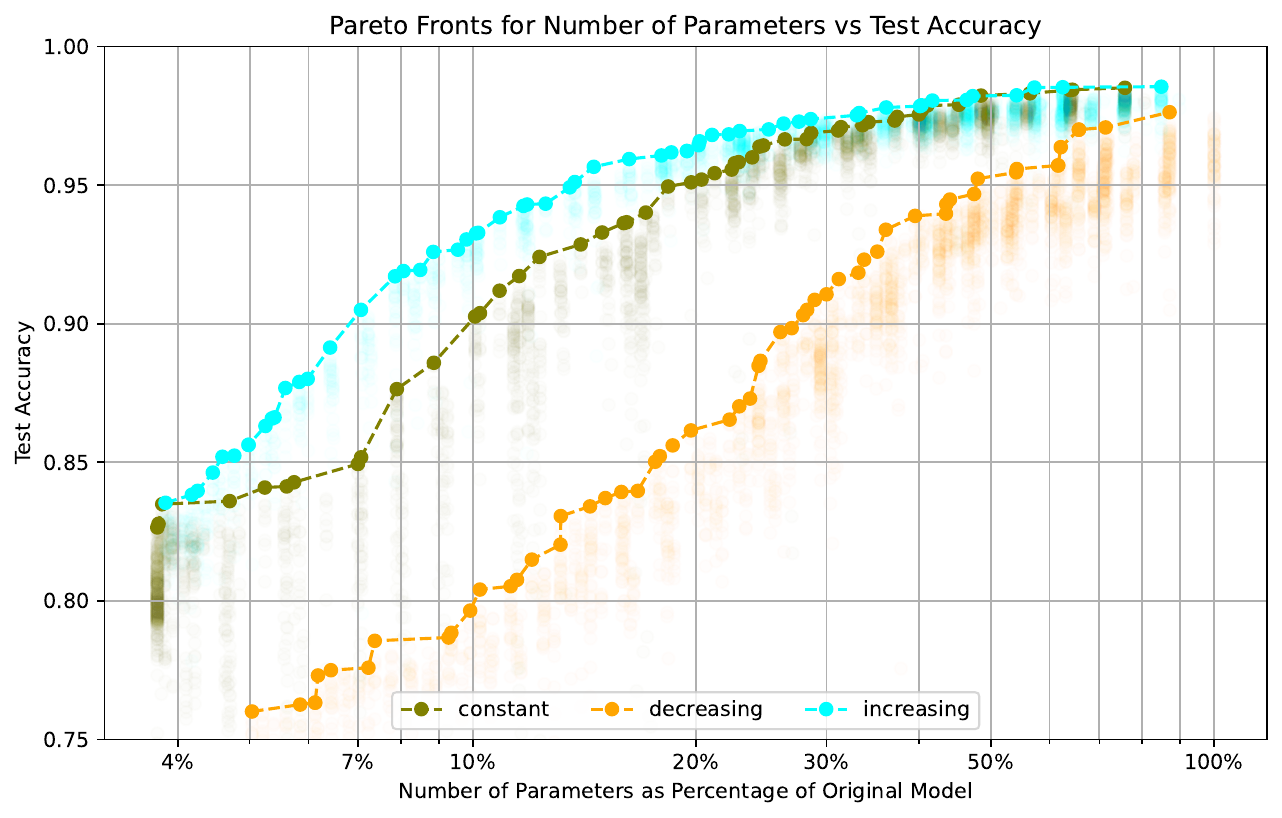} %
    \caption{Ablation study of different bound schedules for CIFAR-10. We find that increasing the amount of error with depth tends to perform best due to simpler trainability.}
    \vspace{-1.5em}
    \label{fig:Pareto}
\end{wrapfigure}
To further emphasize the connection between trainability and compressed performance, we performed an ablation study on various cutoff schedules for the optimal parameter settings. Three schedules were considered: (1) constant with depth, (2) increasing with depth, (3) decreasing with depth.

It is widely known \citep{chen2023which} that earlier layers stabilize more rapidly in training than later layers. Consequently, at each hard thresholding step, earlier layers require more significant adjustments and thus necessitate greater flexibility. We anticipated that the second strategy (increasing with depth) would yield the best performance, while the third strategy (decreasing with depth) would perform the worst. \Cref{fig:Pareto} presents the results of this ablation study, corroborating our hypothesis regarding the dominant influence of trainability on achieving optimal performance for compressed models. In analogy with \citep{RankDyna} we also compared global vs local rank selection strategies in \Cref{ap:global_local}.

\subsection{BERT on GLUE}\label{sec:bert_exp}
We follow \citep{FWSVD,TFWSVD,RankDyna,pletenev2023computational} and test on a subset of tasks in the GLUE benchmark \citep{wang-etal-2018-glue} by compressing a pre-trained BERT model \citep{devlin2019bertpretrainingdeepbidirectional} (104M parameters) using various approaches, with a visualization in \Cref{fig:BERT}. 
We argue that the setting considered here is more akin to zero-shot comparison, as performance of models either quickly diverges, or converges to a non-zero validation loss even without compression when trained for more than $10$ epochs, a phenomenon well-documented in BERT finetuning \citep{mosbach2021on}. This limitation in training duration significantly impacts iterative compression techniques. Specifically, the minimal difference observed between the Pareto fronts of iterative and non-iterative methods in \Cref{fig:BERT} arises because iterative approaches, which rely on refining the model over multiple training iterations, effectively ``massaging'' into a reduced form, become ineffective when prolonged training leads to a decline in test metrics. Consequently, the superior performance of non-iterative methods like FWSVD and TFWSVD in this specific context is primarily attributable to this constraint on training duration. While in scenarios where extended and stable training is feasible, iterative compression methods generally outperform their non-iterative counterparts.

\section{Outlooks}
This paper explored model compression via information geometry and iterative optimization. Our analysis revealed that many compression techniques can be interpreted as approximating information divergences when projecting models onto lower-cost ones, with better approximations resulting in improved zero-shot performance. Our findings suggest however that the choice of projection matters less than the optimization process itself: iterative compression during training plays the dominant role in achieving strong performance post-compression. We note however that in this paper, we have primarily focused on low rank factorization methods applied to medium-sized models, with more complex structures, and approaches overcoming the problem of trainability left for future work.

\bibliography{iclr25/iclr2025_conference}
\bibliographystyle{abbrv} 
\appendix
\section{Appendix: Mathematical Step}
Here, we recap some of the basic notions, primarily following \cite{boct2016inertial}. We denote the Euclidean scalar product and norm as $\langle\cdot,\cdot\rangle$ and $\|\cdot\|$. The finite-dimensional spaces considered here are endowed with the topology induced by the Euclidean norm.

The domain of the function $f: \mathbb{R}^m \rightarrow(-\infty,+\infty]$ is defined by $\operatorname{dom} f=\left\{x \in \mathbb{R}^m\right.$ : $f(x)<+\infty\}$. We say that $f$ is proper if $\operatorname{dom} f \neq 0$. Consider $f: \mathbb{R}^m \rightarrow$ $(-\infty,+\infty]$ proper and lower semicontinuous function. If $x \in \operatorname{dom} f$, we consider the Fr\'echet  subdifferential of $f$ at $x$ as the set
$$
\hat{\partial} f(x)=\left\{v \in \mathbb{R}^m: \liminf _{y \rightarrow x} \frac{f(y)-f(x)-(v, y-x)}{\|y-x\|} \geq 0\right\} .
$$
For $x \notin \operatorname{dom} f$ we set $\partial f(x):=0$. The limiting (Mordukhovich) subdifferential is defined at $x \in \operatorname{dom} f$ by
$$
\partial f(x)=\left\{v \in \mathbb{R}^m: \exists x_n \rightarrow x, f\left(x_n\right) \rightarrow f(x) \text { and } \exists v_n \in \partial f\left(x_n\right), v_n \rightarrow v \text { as } n \rightarrow+\infty\right\}
$$

while for $x \notin \operatorname{dom} f$, one takes $\partial f(x):=\emptyset$. For $f$ convex, both coincide with the convex subdifferential $\partial f(x)=\left\{v \in \mathbb{R}^m: f(y) \geq f(x)+\langle v, y-x\rangle \forall y \in \mathbb{R}^m\right\}$ for all $x \in \operatorname{dom} f$.

If $x \in \mathbb{R}^m$ is a local minimizer of $f$, then $0 \in \partial f(x)$. Notice that in case $f$ is continuously differentiable around $x \in \mathbb{R}^m$ we have $\partial f(x)=\{\nabla f(x)\}$. Let us denote by
$$
\operatorname{crit}(f)=\left\{x \in \mathbb{R}^m: 0 \in \partial f(x)\right\}
$$
the set of (limiting)-critical points of $f$. 
We now define functions satisfying the Kurdyka-Lojasiewicz property, playing a crucial role in proving the convergence of algorithms in the nonconvex setting. For $\eta \in(0,+\infty]$, we denote by $\Theta_\eta$ the class of concave and continuous functions $\varphi:[0, \eta) \rightarrow[0,+\infty)$ such that $\varphi(0)=0, \varphi$ is continuously differentiable on $(0, \eta)$, continuous at 0 and $\varphi^{\prime}(s)>0$ for all $s \in(0, \eta)$.We denote the distance function to a set, defined for $A \subseteq \mathbb{R}^m$ as dist $(x, A)=\inf _{y \in A}\|x-y\|$ for all $x \in \mathbb{R}^m$.

\begin{definition}[Kurdyka-\L ojasiewicz property]
Let $f: \mathbb{R}^m \rightarrow(-\infty,+\infty]$ be a proper and lower semicontinuous function. We say that $f$ satisfies the Kurdyka-\L ojasiewicz (K\L ) property at $\bar{x} \in \operatorname{dom} \partial f=\left\{x \in \mathbb{R}^m: \partial f(x) \neq \emptyset\right\}$ if there exists $\eta \in(0,+\infty]$, a neighborhood $U$ of $\bar{x}$ and a function $\varphi \in \Theta_\eta$ such that for all $x$ in the intersection
$$
U \cap\left\{x \in \mathbb{R}^m: f(\bar{x})<f(x)<f(\bar{x})+\eta\right\}
$$
the following inequality holds
$$
\varphi^{\prime}(f(x)-f(x)) \operatorname{dist}(0, \partial f(x)) \geq 1
$$
If $f$ satisfies the K\L property at each point in dom $\partial f$, then $f$ is called a K\L~function.
\end{definition}
To the class of K\L functions belong semi-algebraic, real sub-analytic, semiconvex, uniformly convex and convex functions satisfying a growth condition. See \citep{boct2016inertial} for a full discussion. The following lemma turns out to be useful characterization of the property:
\begin{lemma}
Let $\Omega \subseteq \mathbb{R}^m$ be a compact set and let $f: \mathbb{R}^m \rightarrow(-\infty,+\infty)$ be a proper and lower semicontinuous function. Assume that $f$ is constant on $\Omega$ and $f$ satisfies the K\L~ property at each point of $\Omega$. Then there exist $\varepsilon, \eta>0$ and $\varphi \in \Theta_\eta$ such that for all $\bar{x} \in \Omega$ and for all $x$ in the intersection
$$
\left\{x \in \mathbb{R}^m: \operatorname{dist}(x, \Omega)<\varepsilon\right\} \cap\left\{x \in \mathbb{R}^m: f(\bar{x})<f(x)<f(\bar{x})+\eta\right\}
$$
the following inequality holds
$$
\varphi^{\prime}(f(x)-f(x)) \operatorname{dist}(0, \partial f(x)) \geq 1
$$   
\end{lemma}

\subsection{Proof of \Cref{prop:convergence}}\label{sec:proof_prop}
\begin{prop}
    Assume that for all $n$, $0<\underline{\alpha}\leq\alpha_n\leq\frac{\sigma_n}{L_{\nabla_{\mc{L}}}}$ and assumptions above hold. Then, the following are true:
    \begin{itemize}
        \item $\left(\mathcal{L} + \lambda \operatorname{rank}\right)(W_n)$ is non-increasing and convergent.
        \item $\sum_n\|W_{n+1}-W_{n}\| < +\infty$
        \item $W_n$ converges, with $\underset{n\to\infty}{\lim} W_n = W^*\in {\operatorname{crit}{\left(\mathcal{L} + \lambda \operatorname{rank}\right)}}$.
        \item For $F_n(W) = \frac12\|\mc{I}_n^{1/2}W\|^2_2$ with $\mc{I}_n\to \mc{I}$ positive definite, the smallest non-zero singular value satisfies $\sigma_{\text{min}}\left(\mc{I}^{1/2}W^*\right)\geq \sqrt{\underline{\alpha} \lambda}$. And thus $\sigma_{\text{min}}\left(W^{*}\right)  \geq \sqrt{\sigma_{\text{max}}\left(\mc{I}\right)\underline{\alpha} \lambda}$
    \end{itemize}
\end{prop}
\begin{proof}
Using \Cref{lem:rank}, rank is semi-algebraic and therefore sub-analytic. Thus $\left(\mathcal{L} + \lambda \operatorname{rank}\right)$ is sub-analytic and therefore K\L. As rank is not continuous, the overall function may not admit a K{\L} exponent.
For a constant $F_n$ this would be a direct application of theory in \cite[Theorem 12]{boct2016inertial} with $\beta = 0$. In general however, the main descent result is as follows:
$$
D_{F_n}\left(W_{n+1}, W_{n}\right)+\left\langle W_{n+1}-W_{n}, \alpha_n \nabla \mc{L}\left(W_{n}\right)\right\rangle+\alpha_n \lambda\operatorname{rank}\left(W_{n+1}\right) \leq \alpha_n \lambda\operatorname{rank}\left(W_{n}\right) .
$$
On the other hand, by descent lemma we have
$$
\left\langle\nabla \mc{L}\left(W_{n}\right), W_{n+1}-W_{n}\right\rangle \geq \mc{L}\left(W_{n+1}\right)-\mc{L}\left(W_{n}\right)-\frac{L_{\nabla \mc{L}}}{2}\left\|W_{n}-W_{n+1}\right\|^2.
$$
At the same time for any $\mu > 0$,
$$
\left\langle W_{n+1}-W_{n}, W_{n-1}-W_{n}\right\rangle \geq-\left(\frac{\mu}{2}\left\|W_{n}-W_{n+1}\right\|^2+\frac{1}{2 \mu}\left\|W_{n-1}-W_{n}\right\|^2\right)
$$
while by assumption
$$
\frac{\sigma_n}{2}\left\|W_{n+1}-W_{n}\right\|^2 \leq D_{F_n}\left(W_{n+1}, W_{n}\right).
$$
This implies that above becomes
$$
(\lambda\operatorname{rank}+\mc{L})\left(W_{n+1}\right)+\frac{\sigma_n-L_{\nabla \mc{L}} \alpha_n}{2 \alpha_n}\left\|W_{n+1}-W_{n}\right\|^2 \leq(\lambda\operatorname{rank}+\mc{L})\left(W_{n}\right).
$$
By the step size assumption, we have $\sigma_n \geq \underline{\alpha}/L_{\nabla\mc{L}}$. By the lipshitzness assumption, we must further have $\sigma_n \leq \overline{L_{\nabla F}}$. Thanks to these bounds, the rest of the results from \cite{boct2016inertial} apply directly.
    
For the last point, this is a direct consequence of \cite[Theorem 3]{hiriart2013eckart}. To be precise, we must have that 
$\sigma_{\text{min}}\left(\mc{I}_n^{1/2}W_{n+1}\right)\geq \sqrt{{\alpha_n} \lambda}$. By continuity of singular values, using e.g. Weyl's inequality, we have the desired result as then the minimum non-zero singular value is continuous outside of zero. Lastly, 
$\sigma_{\text{min}}\left(W^{*}\right)=\sigma_{\text{min}}\left(\mc{I}^{-1/2} \mc{I}^{1/2}W^{*}\right)\geq\sigma_{\text{min}}\left(\mc{I}^{-1/2}\right)\sigma_{\text{min}}\left(\mc{I}^{1/2}W^{*}\right)$
\end{proof}
\subsection{Proof of \Cref{lem:submersion}}\label{ap:proof_lemma}
\begin{lemma}\label{lem:submersion}
    Let $f:X\to Y$ a submersion between two manifolds, and let $M\in Y$ a subset of $Y$. Then, $M$ is an embedded submanifold of $Y$ if and only if $f^{-1}(M)$ is an embedded submanifold of $X$.
\end{lemma}
\begin{proof}
    One way comes directly from  \citep[Corollary 6.31]{lee2003smooth}. For the other direction, we can use the constant rank theorem, which gives us charts $U\subseteq X \to V\subseteq Y$, such that $f:U\to V$ is $(x_1,\dots,x_m)\to(x_1,\dots,x_n)$, where $m=\dim X, n=\dim Y$ with $m\geq n$. Then $f^{-1}({M})\subseteq X$, so $f^{-1}({M})\cap U \subseteq V$ embedded. Then find slice chart for $f^{-1}(M)\cap U$, which gives us a slice chart for $M\cap V$.
\end{proof}
\section{Methods Overview}\label{ap:methods}

\begin{table}
    \centering
    \caption{Summary of Compression Methods}
    \resizebox{\textwidth}{!}{%
        \begin{tabular}{lllcc}
            \toprule
            Projection & Method & Origin & Iterative & Rank Selection Criterion \\
            \midrule
            \multirow{6}{*}{Euclidean} & \textcolor{rgb,255:red,178;green,34;blue,34}{OIALR} & \citep{OIALR} & Yes & Layer Maximal Singular Value \\
            & \textcolor{rgb,255:red,255;green,140;blue,0}{IEHT} &  (Ours \Cref{sec:IHT,ap:IHT})& Yes & Layer Energy \\
            & \textcolor{rgb,255:red,128;green,128;blue,0}{TRP} & \citep{TRP} & Yes & Layer Energy \\
            & \textcolor{rgb,255:red,255;green,99;blue,71}{globalIEHT} & (Ours \Cref{sec:IHT,ap:IHT})& Yes & Total Energy \\ 
            & \textcolor{rgb,255:red,139;green,69;blue,19}{SVD} & \citep{psichogios1994svd} & No & Fixed \\
            & \textcolor{rgb,255:red,210;green,105;blue,30}{TTM} & \citep{oseledets2011tensor} & No & Fixed \\
            \midrule
            \multirow{6}{*}{Information} & \textcolor{rgb,255:red,128;green,0;blue,128}{IFHT} & (Ours \Cref{sec:IHT,ap:IHT}) & Yes & Layer 'Fisher' Energy \\
            & \textcolor{rgb,255:red,17;green,61;blue,128}{FWTRP} & (Ours \Cref{sec:FWTRP}) & Yes & Layer 'Fisher' Energy \\
            & \textcolor{rgb,255:red,138;green,43;blue,226}{globalIFHT}  & (Ours \Cref{sec:IHT,ap:IHT}) & Yes & Global 'Fisher' Energy \\
            & \textcolor{rgb,255:red,65;green,105;blue,225}{FWSVD} & \citep{FWSVD} & No & Fixed \\
            & \textcolor{rgb,255:red,30;green,144;blue,255}{TFWSVD} & \citep{TFWSVD} & No & Fixed \\
            & \textcolor{rgb,255:red,186;green,85;blue,211}{FWTTM} & \citep{pletenev2023computational} & No & Fixed \\
            \bottomrule
        \end{tabular}
    }
    \label{tab:method_comparison}
    \vspace{-1em}
\end{table}
\subsection{OIALR}
\citep{OIALR} observed that the orthogonal basis of a network's weights stabilizes during the training process. Based on this, they propose to reduce the number of trainable parameters by iteratively orthogonalizing as follows:
\begin{enumerate}[nosep]
    \item The network is initially trained with a traditional full-rank scheme.
    \item After a number of iterations, the network's matrix weights are transitioned to their $U\Sigma V^T$ representation using singular value decomposition (SVD).
    \item The orthogonal bases $U$ and $V^T$ are frozen, and only the square matrix $\Sigma$ is trained using backpropagation.
    \item After a specified number of training steps, the bases $U$ and $V^T$ are updated by extracting the new bases from the trained $\Sigma$ matrix using SVD.
    \item A new inner rank is found by removing the singular values from $\Sigma$ whose absolute magnitude is less than $\beta$ times the largest singular value in the current $\Sigma$. $\beta$ is a hyperparameter that defaults to 0.1.
    \item This process is repeated until the end of training.
\end{enumerate}
We summarize the method in \Cref{alg:oialr}, and attempt to formalize the approach below.
\subsubsection{Orthogonality Informed Low Rank Training}\label{sec:app_oialr_theory}
\Cref{prop:convergence} establishes convergence and a maximal rank of the underlying solution. However, in practice using such a method implies computing SVD of all the weights at each iteration, making it prohibitively expensive. Below we instead consider the specific algorithm of \cite{OIALR}. For $U^0 = V^0 = \operatorname{Id}$ and $S^0 = W^0$, we can summarize it as:
\begin{align*}
    &S_*^{k+1} \in \crit_S \mc{L}(U_kSV_k^\top) \\
    &S^{k+1} \in \operatorname{Prox}_{\lambda_k}(\operatorname{rank})(S_*^{k+1}) \\
    &U_{k+1}, V_{k+1}^{\top} \in \operatorname{SVD}(U_kS^{k+1}V_k^\top) \\ 
    &W^{k+1} = U_kS^{k+1}V_k 
\end{align*}
This effectively means that we iteratively project via the prox operator, followed by finding a critical point on the low-rank submanifold. Note however, success of this method relies on the fact that during training the basis elements stabilize quite quickly. This allows us to choose a local chart on the low-rank submanifold on which we can perform the desired projection. 

In general it would be rather difficult to say anything about convergence of such an algorithm, as even in the convex case, given a small misalignment of the linear chart $U_k, V_k$ with the chart at the global minimum, the global minimum becomes unattainable. 

However, we can instead control directly the loss value jump under projection. So, say that instead of $S_*^{k+1} \in \crit_S \mc{L}(U_kSV_k^\top)$, we perform only a single step of gradient descent. In this case we have that 
$S_*^{k+1} = S^k - \frac{1}{L_{\nabla_{\mc{L}}}}\nabla_S\mc{L}(U_kS^kV_k^\top)$. However we know that $\nabla_S\mc{L}(U_kSV_k^{\top}) = U_k^{\top}\nabla_W\mc{L}(U_kSV_k^{\top})V_k$.
By descent lemma we have that:
\[
\mc{L}(U_kS^{k+1}_*V_k^\top) - \mc{L}(U_kS^kV_k^\top) \leq - \frac{L_{\nabla_{\mc{L}}}}{2}\|S^{k+1}_*-S^{k}\|^2 = -\frac{1}{2L_{\nabla_{\mc{L}}}}\|\nabla_S\mc{L}(U_kS^kV_k^\top)\|^2
\]
Thus, gradient descent guarantees decrease of the objective function. What about the projection step? Assuming now again that $S_*^{k+1} \in \crit_S \mc{L}(U_kSV_k^\top)$, we know that $\nabla_S\mc{L}(U_kS_{*}^{k+1}V_k^\top) = 0$. This however, does not guarantee criticality in the full space $\nabla_W\mc{L}(U_kS_{*}^{k+1}V_k^\top) \neq 0$. Despite this, the descent lemma still provides us with a bound on the total decrease: 
\begin{align*}
\mc{L}(U_kS^{k+1}V_k^\top) &- \mc{L}(U_kS^{k+1}_*V_k^\top) \\ &\leq \langle \nabla_W \mc{L}(U_kS^{k+1}_*V_k^\top),U_kS_{*}^{k+1}V_k^\top-U_kS^{k+1}V_k^\top\rangle \\&\qquad\qquad\qquad\qquad\qquad\qquad\qquad\qquad\;\qquad+ \frac{L_{\nabla_{\mc{L}}}}{2}\|U_kS^{k+1}_*V_k^\top-U_kS^{k+1}V_k^\top\|^2 \\    
&=\langle U_k^{\top}\nabla_W \mc{L}(U_kS^{k+1}_*V_k^\top)V_k,S_{*}^{k+1}-S^{k+1}\rangle + \frac{L_{\nabla_{\mc{L}}}}{2}\|U_kS^{k+1}_*V_k^\top-U_kS^{k+1}V_k^\top\|^2 \\
&=\langle \nabla_S \mc{L}(U_kS^{k+1}_*V_k^\top),S_{*}^{k+1}-S^{k+1}\rangle + \frac{L_{\nabla_{\mc{L}}}}{2}\|S^{k+1}_*-S^{k+1}\|^2 = \\
&= \frac{L_{\nabla_{\mc{L}}}}{2}\|S^{k+1}_*-S^{k+1}\|^2\\
&\leq \lambda_k L_{\nabla_{\mc{L}}}\left\{\operatorname{rank}(S^{k+1}_*) - \operatorname{rank}(S^{k+1})\right\} 
\\&< \sigma^2_{\text{min}}\left(S^{k+1}\right)L_{\nabla_{\mc{L}}}\left\{\operatorname{rank}(S^{k+1}_*) - \operatorname{rank}(S^{k+1})\right\}
\end{align*}
Thus, as long as $\|S^{k+1}_*-S^{k+1}\|^2$ is bounded, the objective increase is not large. This is directly controlled through $\lambda_k$, as by definition of prox, we have $\frac{1}{2\lambda_k}\|S^{k+1}-S_*^{k+1}\|^2 \leq \operatorname{rank}(S^{k+1}_*) - \operatorname{rank}(S^{k+1})$. We also know that $\lambda_k$ controls the largest singular value cutoff, and thus we have that $\lambda_k < \sigma^2_{\text{min}}\left(S^{k+1}\right)$. Although it is worth emphasizing that normally in the projection step we have explicit control on the size $\|S^{k+1}_*-S^{k+1}\|^2$. 
\subsection{IEHT and IFHT}\label{ap:IHT}
Based on the observations in \Cref{sec:exp_vit}, and section above we realize that the cutoff selection proposed by \citep{OIALR} is detrimental to downstream performance, especially given that compatibility the initial choice of chart is fundamental for establishing convergence. We instead propose to perform cutoffs based on the total energy (sum of squares of singular values), which turns out to be much milder.
\begin{figure*}[ht]
\centering
\begin{tabular}{cc}
\scalebox{0.45}{ %
\begin{minipage}[t]{\linewidth} %
\begin{algorithm}[H]
\caption{OIALR, {\textcolor{red}{IEHT}}}\label{alg:oialr}
\KwData{Model $M$}
\Parameters{Training steps $t_{\max}$, delay steps $d$, low-rank update frequency $\nu$, singular value cutoff fraction $\beta$}
\KwResult{Compressed trained model.} 
\For{$t = 0 \dots t_{\max}$}{
\uIf{$t<d$}{
    Train full-rank \;
  }
  \uElseIf{$t = d$}{
    Convert network to $USV^\top$ representation \;
    Freeze $U$ and $V$\;
  }
  \uElseIf{$t = 0\text{ mod }\nu$}{
    \For{$i = 0 \dots L$}{
${U_i}^{\prime}, {S_i}^{\prime}, {{V_i}^{\prime}}^\top \leftarrow \operatorname{svd}({S_i})$ \;
    $U_i\gets U_iU^{\prime}_i$ \;
    $V_i\gets V_iV^{\prime}_i$ \;
    $S_i\gets S^{\prime}_i$ \;
    Remove singular values $<\beta \cdot \max \left({S}_i\right)$ \;
    \textcolor{red}{For IEHT: Remove lower singular values contributing $<\beta \%$ of the total energy}\;
    Reshape ${U}_i, {V}_i, {S}_i$, and optimizer states\;
  }
    
  }
  \Else{
    Train network in ${U} S {V}^\top$ representation, with $U$ and $V$ frozen\;
  }
}
\texttt{\# Now compile into smaller model} \\
\For{$i = 0 \dots L$}{
${U_i}^{\prime}, {S_i}^{\prime}, {{V_i}^{\prime}}^\top \leftarrow \operatorname{svd}({S_i})$ \;
    $U_i\gets U_iU^{\prime}_i\sqrt{S^{\prime}_i}$ \;
    $V_i\gets V_iV^{\prime}_i\sqrt{S^{\prime}_i}$ \;
    Remove ${S}_i$ as parameters.
  }
\Return Trained factorized model $M$
\end{algorithm}
\end{minipage}
}
&

\scalebox{0.45}{ %
\begin{minipage}[t]{\linewidth} %
\begin{algorithm}[H]
\caption{IFHT}\label{alg:fwoialr}
\KwData{Model $M$}
\Parameters{Training steps $t_{\max}$, delay steps $d$, low-rank update frequency $\nu$, \textcolor{red}{energy cutoff fraction $\beta$}, \textcolor{red}{dataset for fisher estimation $\mathcal{D}$}}
\KwResult{Compressed trained model.} 
\For{$t = 0 \dots t_{\max}$}{
\uIf{$t<d$}{
    Train full-rank \;
  }
  \uElseIf{$t = d$}{
    Convert network to $USV^\top$ representation \;
    Freeze $U$ and $V$\;
  }
  \uElseIf{$t = 0\text{ mod }\nu$}{
    \textcolor{red}{${\widetilde{\mathcal{I}}}$ = $\texttt{fisher}(M, \mathcal{D})$\;}
    \For{$i = 0 \dots L$}{\textcolor{red}{
${U_i}^{\prime}, {S_i}^{\prime}, {{V_i}^{\prime}}^\top \leftarrow \operatorname{svd}({\widetilde{\mathcal{I}}}{S_i})$ \;
    $U_i\gets U_i{\widetilde{\mathcal{I}}}^{-1}U^{\prime}_i$ \;}
    $V_i\gets V_iV^{\prime}_i$ \;
    $S_i\gets S^{\prime}_i$ \;
    \textcolor{red}{Remove lower singular values contributing less than $<\beta \%$ of the total energy\;
    Re-factorize to keep $U_i,V_i$ semi-orthogonal\;}
    Reshape ${U}_i, {V}_i, {S}_i$ and optimizer states\;
  }
    
  }
  \Else{
    Train network in ${U} S {V}^\top$ representation, with $U$ and $V$ frozen\;
  }
}
\texttt{\# Now compile into smaller model} \\
\For{$l = 0 \dots L$}{
${U_i}^{\prime}, {S_i}^{\prime}, {{V_i}^{\prime}}^\top \leftarrow \operatorname{svd}({S_i})$ \;
    $U_i\gets U_iU^{\prime}_i\sqrt{S^{\prime}_i}$ \;
    $V_i\gets V_iV^{\prime}_i\sqrt{S^{\prime}_i}$ \;
    Remove ${S}_i$ as parameters.
  }
\Return Trained factorized model $M$
\end{algorithm}
\end{minipage}
}
\end{tabular}
\end{figure*}
\subsection{TRP and FWTRP}\label{sec:FWTRP}

\begin{figure*}[ht]
\centering
\begin{tabular}{cc}
\scalebox{0.45}{ %
\begin{minipage}[t]{\linewidth} %
\begin{algorithm}[H]
\caption{TRP (Trained Rank Pruning)}\label{alg:trp}
\KwData{Model $M$}
\Parameters{Training steps $t_{\max}$, low-rank update frequency \texttt{trp\_frequency}, singular value cutoff fraction $\beta$}
\KwResult{Compressed trained model.} 
\For{$t = 0 \dots t_{\max}$}{
\uIf{$t\;\operatorname{mod} $ \texttt{trp\_frequency} $ \neq 0$}{
    Train full-rank \;
  }
  \uElse{
    \For{$i = 0 \dots L$}{
    ${U_i}, {S_i}, {{V_i}}^\top \leftarrow \operatorname{svd}({W_i})$ \;
    Remove lower singular values contributing less than $<\beta \%$ of the total energy at $k$\;
    $W_i\gets U_i[:,:k]S[:k,:k]V_i[:,:k]^\top$ \;
    $G_i \gets U_i[:,:k]V_i[:,:k]^\top$
  }
  }
  \If{$t\;\operatorname{mod}$ \texttt{nuclear\_norm\_frequency} $ = 0$}{
    Make gradient step with \texttt{nuclear\_norm\_weight}$*G_i$
  }
}
\texttt{\# Now compile into smaller model} \\
\For{$i = 0 \dots L$}{
${U_i}, {S_i}, {{V_i}}^\top \leftarrow \operatorname{svd}({W_i})$ \;
    $U_i\gets U_i\sqrt{S_i}$ \;
    $V_i\gets V_i\sqrt{S_i}$ \;
    Remove ${S}_i$ as parameters.
  }
\Return Trained factorized model $M$
\end{algorithm}
\end{minipage}
}
&

\scalebox{0.45}{ %
\begin{minipage}[t]{\linewidth} %
\begin{algorithm}[H]
\caption{FWTRP (Trained Rank Pruning)}\label{alg:fwtrp}
\KwData{Model $M$}
\Parameters{Training steps $t_{\max}$, low-rank update frequency \texttt{trp\_frequency}, regularization frequency \texttt{nuclear\_norm\_frequency}, singular value cutoff fraction $\beta$}
\KwResult{Compressed trained model.} 
\For{$t = 0 \dots t_{\max}$}{
\uIf{$t\;\operatorname{mod} $ \texttt{trp\_frequency} $ \neq 0$}{
    Train full-rank \;
  }
  
  \uElse{
     \textcolor{red}{${\widetilde{\mathcal{I}}}$ = $\texttt{fisher}(M, \mathcal{D})$\;}
    \For{$i = 0 \dots L$}{
   \textcolor{red}{${U_i}, {S_i}, {{V_i}}^\top \leftarrow \operatorname{svd}({\widetilde{\mathcal{I}}}{S_i})$ }\;
    Remove lower singular values contributing less than $<\beta \%$ of the total energy at $k$\;
    \textcolor{red}{$W_i\gets {\widetilde{\mathcal{I}}}^{-1}U_i[:,:k]S[:k,:k]V_i[:,:k]^\top$ }\;
    $G_i \gets U_i[:,:k]V_i[:,:k]^\top$
  }
  }
  \If{$t\;\operatorname{mod}$ \texttt{nuclear\_norm\_frequency} $ = 0$}{
    Make gradient step with \texttt{nuclear\_norm\_weight}$*G_i$
  }
}
\texttt{\# Now compile into smaller model} \\
\For{$i = 0 \dots L$}{
${U_i}, {S_i}, {{V_i}}^\top \leftarrow \operatorname{svd}({W_i})$ \;
    $U_i\gets U_i\sqrt{S_i}$ \;
    $V_i\gets V_i\sqrt{S_i}$ \;
    Remove ${S}_i$ as parameters.
  }
\Return Trained factorized model $M$
\end{algorithm}
\end{minipage}
}
\end{tabular}
\end{figure*}

\bigskip

\noindent \textbf{Description of TRP:}
In trained rank pruning, every \texttt{trp\_frequency} steps, the networks singular values are hard thresholded to $0$, with nuclear norm gradients saved, and subsequently each \texttt{nuclear\_norm\_frequency}, gradient steps are performed.

\noindent \textbf{Description of FWTRP:}
The overall description is the same as that of TRP, but with SVD now being performed with respect to approximated fisher information.

\section{Further Experimental Details}\label{ap:details}
For CIFAR10, the models were trained for 250 epochs, while for ImageNet for 50. For BERT, models were trained for 100 epochs, even though convergence was observed very quickly. Pre-trained models were acquired from the \texttt{timm} and \texttt{huggingface} libraries and utilized as the baseline, with various compression techniques applied. 

In order to create the visualizations on \Cref{fig:BERT,fig:cifar10,fig:imagenet}, all the algorithms of \Cref{ap:methods} were run on various parrameter settings. For each run, test accuracy was evaluated at the end of each epoch and plotted on the figure. Then, the resulting pareto front was plotted to showcase best possible performance amongst all the parameters. The parameter values for all methods are shown in \Cref{tab:bert_params,tab:imagenet_params,tab:cifar10_params}. Each experimental setting was run on a single 24GB GPU for up to 1 week of execution.
\begin{table}[h!]
    \centering
    \caption{Parameter variations for ViT on CIFAR10}
    \label{tab:cifar10_params}
    \resizebox{\textwidth}{!}{%
    \begin{tabular}{@{}lll@{}}
        \toprule
        Method & Parameter & Explored Values \\
        \midrule
        SVD, FWSVD, TFWSVD & \texttt{rank} & 8, 16, 24, 32, 64, 128, 256 \\
        \midrule
        TRP, FWTRP & \texttt{trp\_threshold} & 0.9, 0.95, 0.99 \\
        & \texttt{trp\_frequency} & 50, 100, 500, 1000 \\
        & \texttt{nuclear\_norm\_weight} & 0.0003, 0.00005, 0.005 \\
        & \texttt{nuclear\_norm\_frequency} & \texttt{trp\_frequency} // 2 \\
        \midrule
        OIALR & \texttt{oialr\_threshold} & 0.25, 0.2, 0.15 \\
        & \texttt{oialr\_frequency} & 10, 15, 25 \\
        & \texttt{oialr\_depth\_schedule} & \texttt{'constant'} \\
        & \texttt{oialr\_type} & \texttt{'epoch'} \\
        & \texttt{oialr\_min\_rank\_percent} & 0.02 \\
        \midrule
        IEHT, IFHT & \texttt{oialr\_threshold} & 0.9, 0.925, 0.95, 0.97 \\
        & \texttt{oialr\_frequency} & 10, 15, 25 \\
        & \texttt{oialr\_depth\_schedule} & \texttt{'constant'},\texttt{'increasing'} \\
        & \texttt{oialr\_type} & \texttt{'epoch'} \\
        & \texttt{oialr\_min\_rank\_percent} & 0.02 \\
        \midrule
        Global IHT & \texttt{oialr\_threshold} & 0.8, 0.9, 0.95, 0.99 \\
        & \texttt{oialr\_frequency} & 10, 25, 50 \\
        & \texttt{oialr\_type} & \texttt{'epoch'} \\
        \bottomrule
        Base Training & \texttt{epochs} & 250 \\
        & \texttt{learning\_rate} & 0.0001 \\
        & \texttt{batch\_size} & 64 \\
        & \texttt{optimizer} & AdamW \\
    \end{tabular}%
    }
\end{table}

\begin{table}[h!]
    \centering
    \caption{Parameter variations for ViT on ImageNet}
    \label{tab:imagenet_params}
    \resizebox{\textwidth}{!}{%
    \begin{tabular}{@{}lll@{}}
        \toprule
        Method & Parameter & Explored Values \\
        \midrule
        SVD, FWSVD, TFWSVD & \texttt{rank} & 24, 32, 64, 128, 256, 384, 512 \\
        \midrule
        TRP, FWTRP & \texttt{trp\_threshold} & 0.9, 0.95, 0.97, 0.99 \\
        & \texttt{trp\_frequency} & 5000, 10000, 50000 \\
        & \texttt{nuclear\_norm\_weight} & 0.003, 0.05 \\
        & \texttt{nuclear\_norm\_frequency} & \texttt{trp\_frequency} // 2 \\
        \midrule
        IEHT, IFHT & \texttt{oialr\_threshold} & 0.9, 0.925, 0.95, 0.97 \\
        & \texttt{oialr\_frequency} & 2, 4, 8, 10 \\
        & \texttt{oialr\_depth\_schedule} & \texttt{'constant'},\texttt{'increasing'} \\
        & \texttt{oialr\_type} & \texttt{'epoch'} \\
        & \texttt{weight\_decay} & 5e-4 \\
        & \texttt{oialr\_min\_rank\_percent} & 0.02 \\
        \bottomrule
        Base Training & \texttt{epochs} & 50 \\
        & \texttt{learning\_rate} & 0.0001 \\
        & \texttt{batch\_size} & 128 \\
        & \texttt{optimizer} & AdamW \\
    \end{tabular}%
    }
\end{table}

\begin{table}[h!]
    \centering
    \caption{Parameter variations for BERT on GLUE}
    \label{tab:bert_params}
    \resizebox{\textwidth}{!}{%
    \begin{tabular}{@{}lll@{}}
        \toprule
        Method & Parameter & Explored Values \\
        \midrule
        SVD, FWSVD, TFWSVD & \texttt{rank} & 16, 24, 32, 64, 96, 128, 196, 256, 384 \\
        \midrule
        IEHT, IFHT & \texttt{oialr\_threshold} & 0.95, 0.97, 0.98, 0.99 \\
        & \texttt{oialr\_frequency} & 10, 25, 50, 100, 500, 1000 \\
        & \texttt{oialr\_depth\_schedule} & \texttt{'increasing'} \\
        & \texttt{oialr\_type} & \texttt{'step'} \\
        & \texttt{oialr\_min\_rank\_percent} & 0.02 \\
        \midrule
        TRP, FWTRP & \texttt{trp\_threshold} & 0.95, 0.97, 0.98, 0.99 \\
        & \texttt{trp\_frequency} & 10, 50, 100, 500 \\
        & \texttt{nuclear\_norm\_weight} & 0.0003, 0.005 \\
        & \texttt{nuclear\_norm\_frequency} & \texttt{trp\_frequency} // 2 \\
        \bottomrule
        Base Training & \texttt{epochs} & 100 \\
        & \texttt{learning\_rate} & 0.00002 \\
        & \texttt{batch\_size} & 32 \\
        & \texttt{optimizer} & AdamW \\
    \end{tabular}%
    }
\end{table}

\section{Global vs Local}\label{ap:global_local}
\begin{figure}[ht]
    \centering
    \includegraphics[width=\textwidth]{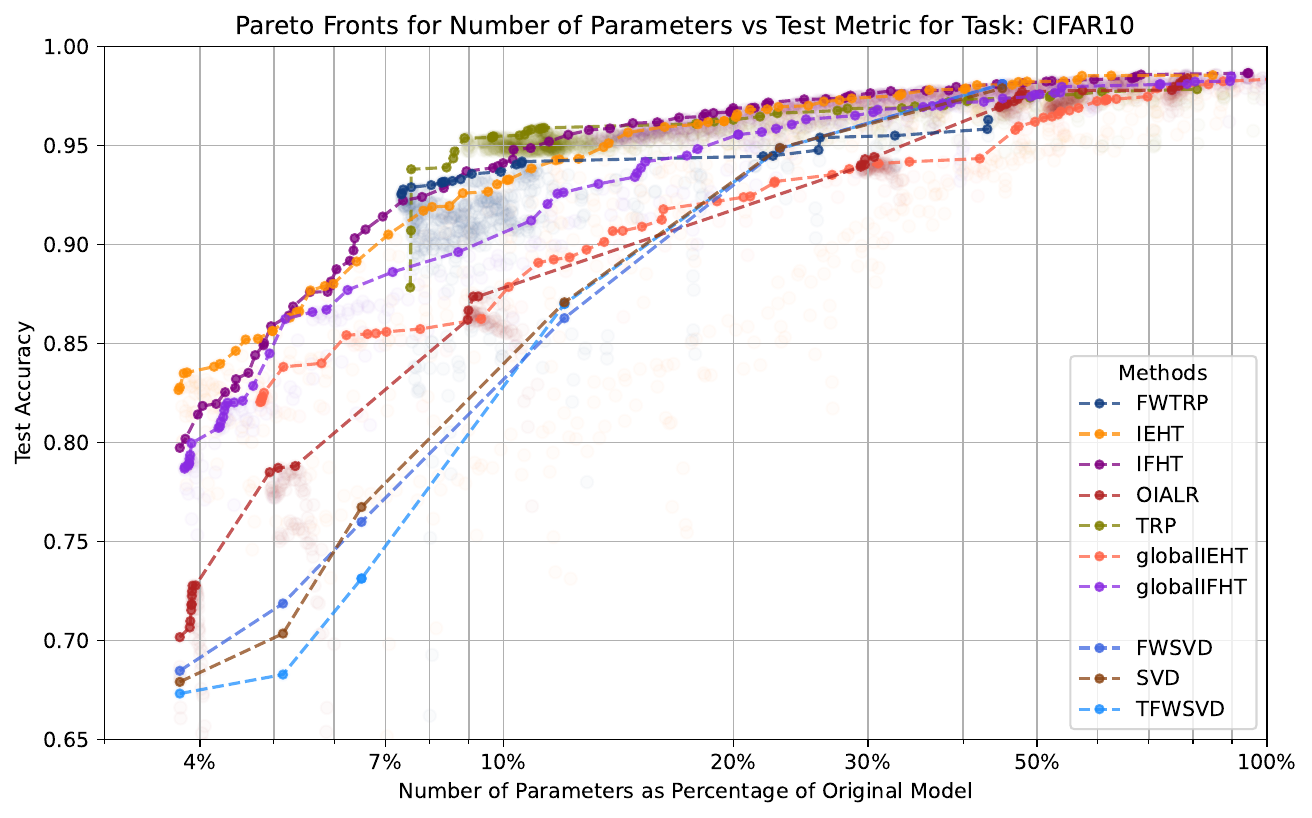} %
        \caption{CIFAR-10 classification, illustrating also the globalIEHT and globalIFHT methods.}
        \label{fig:globallocal}
\end{figure}
Given the results of \citep{RankDyna}, we wanted to see whether a move towards global rank selection in IFHT would result in improved performance in the setting of CIFAR10 with vision transformers. We illustrate the joint plot in \Cref{fig:globallocal}. For the sake of consistency, we illustrate a global modification of IEHT as well. As expected, the globalIEHT method does not perform well, as the singular values differ significantly amongst the layers, resulting in a global selection that is likely to create bottlenecks very early on. However, unlike \citep{RankDyna}, globalIFHT also underperforms, though not significantly so, when compared to the purely local IFHT. We hypothesize that the reason for this is analogous to the observation that BERT finetuning is more analogous to zero-shot performance comparison, as finetuning tends to not generalize, as discussed in \Cref{sec:bert_exp}.

\end{document}